\documentclass[journal]{IEEEtran}%
\usepackage{xcolor,soul,framed} %,caption

\colorlet{shadecolor}{yellow}
\usepackage[pdftex]{graphicx}
\graphicspath{{../pdf/}{../jpeg/}}
\DeclareGraphicsExtensions{.pdf,.jpeg,.png}

\usepackage[cmex10]{amsmath}
\usepackage{array}
\usepackage{amsthm}

\newtheoremstyle{theorem}% name
  {\topsep}% space above
  {\topsep}% space below
  {}% body font
  {}% indent amount
  {\itshape}% theorem head font
  {:}% punctuation after theorem head
  {.5em}% space after theorem head
  {\thmname{#1}\thmnumber{ #2}\thmnote{ (#3)}}% theorem head spec
\theoremstyle{theorem}

\newtheoremstyle{proposition}% name
  {\topsep}% space above
  {\topsep}% space below
  {}% body font
  {}% indent amount
  {\itshape}% theorem head font
  {:}% punctuation after theorem head
  {.5em}% space after theorem head
  {\thmname{#1}\thmnumber{ #2}\thmnote{ (#3)}}% theorem head spec
\theoremstyle{proposition}
\newenvironment{sequation}{\small\begin{equation}}{\end{equation}}

\usepackage{setspace}
\usepackage{mdwmath}
\usepackage{mdwtab}
\usepackage{eqparbox}
\usepackage{url}
\usepackage{amsmath}
\usepackage{amssymb}
\usepackage{amsfonts}
\usepackage{amsmath}
\newtheorem{theorem}{Theorem}
\newtheorem{proposition}{Proposition}
\usepackage{multicol}
\usepackage{bm}
\newcommand{\ie}{\textit{i}.\textit{e}.}
\newcommand{\stt}{\textit{s}.\textit{t}.}
\newcommand{\re}{\mathrm{Re}}
\newcommand{\im}{\mathrm{Im}}
\newcommand{\notation}{\textit{Notation}.}
\usepackage{algorithm}
\usepackage{algorithmicx}
\usepackage{algpseudocode}
\usepackage{cancel}
\usepackage{subeqnarray}
\usepackage{cases}
\usepackage{mathrsfs}
\usepackage{subfigure}
\usepackage[colorlinks,
            linkcolor=black,
            anchorcolor=black,
            citecolor=black
            ]{hyperref}

\ifCLASSOPTIONcompsoc
  \usepackage[nocompress]{cite}
\else
  \usepackage{cite}
\fi
\ifCLASSINFOpdf
\else
\fi

\begin{document}
\title{AN-GCN: An Anonymous Graph Convolutional Network Defend Against Edge-Perturbing Attacks}
%\title{Withdrawing the edge-reading Permission of Graph Convolutional Network to Against Adversarial Attack}

\author{Ao~Liu,~\IEEEmembership{Student Member,~IEEE},
        Beibei~Li,~\IEEEmembership{Member,~IEEE},
        Tao~Li,
        Pan~Zhou,~\IEEEmembership{Senior Member,~IEEE}
        and~Rui~Wang
\IEEEcompsocitemizethanks{
\IEEEcompsocthanksitem This work was supported in part by the National Key Research and Development Program of China (No. 2020YFB1805400) and in part by the National Natural Science Foundation of China (No. U19A2068, No. 62032002, and No. 62002248). \textit{(Corresponding author: Beibei Li)} \IEEEcompsocthanksitem A. Liu, B. Li, and T. Li are with the School of Cyber Science and Engineering, Sichuan University, Chengdu
610065, China (e-mail: liuao@stu.scu.edu.cn; libeibei@scu.edu.cn; litao@scu.edu.cn).
\IEEEcompsocthanksitem P. Zhou is with the Hubei Engineering Research Center on Big Data Security, School of Cyber Science and Engineering, Huazhong University of Science
and Technology, Wuhan 430074, China (e-mail: panzhou@hust.edu.cn).
\IEEEcompsocthanksitem R. Wang is with the Department of Intelligent Systems, Delft University of Technology, Delft 2628XE, Netherlands (e-mail: R.Wang-8@tudelft.nl).

% note need leading \protect in front of \\ to get a newline within \thanks as
% \\ is fragile and will error, could use \hfil\break instead.
}% <-this % stops an unwanted space

}

% The paper headers
\markboth{}%
{Shell \MakeLowercase{\textit{et al.}}: Bare Demo of IEEEtran.cls for Computer Society Journals}

\IEEEtitleabstractindextext{%
\begin{abstract}

%The existing research works do not seem to be able to unify the formulation of such edge-perturbing attacks, which hinders the proposal of defense scheme against all such attacks. Thus, we unify edge-perturbing attacks as an automatic general attack model, named \textit{edge-reading attack} (ERA). ERA can find the concealed and high success rate attack scheme by automatically traverses and perturbs edges repeatedly. ERA is also the unified mathematical formula of edge-perturbing attacks. Relying on ERA, we further demonstrate the vulnerability of GCNs, i.e., the edge-reading permission can easily create opportunities for adversarial attacks. To address this problem, we propose an \textit{anonymous graph convolutional network} (AN-GCN), which allows classifying nodes without reading the edge information of GCNs.
Recent studies have revealed the vulnerability of graph convolutional networks (GCNs) to edge-perturbing attacks, such as maliciously inserting or deleting graph edges. However, a theoretical proof of such vulnerability remains a big challenge, and effective defense schemes are still open issues. In this paper, we first generalize the formulation of edge-perturbing attacks and strictly prove the vulnerability of GCNs to such attacks in node classification tasks. Following this, an \textit{anonymous graph convolutional network}, named AN-GCN, is proposed to counter against edge-perturbing attacks. Specifically, we present a node localization theorem to demonstrate how the GCN locates nodes during its training phase. In addition, we design a staggered Gaussian noise based node position generator, and devise a spectral graph convolution based discriminator in detecting the generated node positions. Further, we give the optimization of the above generator and discriminator. AN-GCN can classify nodes without taking their position as input. It is demonstrated that the AN-GCN is secure against edge-perturbing attacks in node classification tasks, as AN-GCN classifies nodes without the edge information and thus makes it impossible for attackers to perturb edges anymore. Extensive evaluations demonstrated the effectiveness of the general edge-perturbing attack model in manipulating the classification results of the target nodes. More importantly, the proposed AN-GCN can achieve 82.7\% in node classification accuracy without the edge-reading permission, which outperforms the state-of-the-art GCN.

% the proposed AN-GCN can achieve outperformed high accuracy in node classification compared to general GCNs, while maintaining high-level security in defending against adversarial attacks.

% The core principle of AN-GCN is adversarial generate. AN-GCN will randomly generate a graph structure for each training epoch, until it can approximate the underlying true node position distribution well. 

% we make nodes participate in prediction anonymously, thus defense . Specifically, a general attack model for the GCN is first given to prove that exposure of node position leads to the vulnerability of graph. Second, we find how GCN locating node position by deconstructing the input of the model and building the node signal model. Last, we put a generative adversarial network (GAN) into the GCN to generate node positions randomly in noise, used for anonymize nodes thus prevent exposure of node position. Extensive evaluations demonstrate the proposed reason of graph vulnerability confirmed as correct according to the effectiveness of the proposed general attack model, and the proposed AN-GCN can successfully anonymize nodes while keeping message  passing on graphs, thus can effectively defend against such attacks while ensuring the accuracy of node prediction.

\end{abstract}

% A significant reason  is that the nodes position is exposed to GCNs, resulting in that perturbations can be aggregated to the target nodes through the topology of the graph. 
% Note that keywords are not normally used for peerreview papers.
}

% make the title area
\maketitle

\IEEEdisplaynontitleabstractindextext

\IEEEpeerreviewmaketitle

\section{Introduction}\label{sec:introduction}
%%%%%%%%%%%%%%%%%%%%%%%%%%%%%%%%%%%%%%%%%%%%%%%%%%%%%%%%%%%%%%%%%%%

\IEEEPARstart{G}{raph} is the core for many high impact applications ranging from the analysis of social networks, over gene interaction networks, to interlinked document collections. In many tasks related to the graph structure, node classification, predicting labels of unknown nodes based on a small number of labeled known nodes, is always a hot issue. However, vulnerabilities of predicting the node category have been gradually explored. Only slight deliberate perturbations of nodes’ features or graph structure can lead to completely wrong predictions, which leads to a critical issue in many application areas, including those where adversarial perturbations can undermine public trust~\cite{95}, interfere with human decision making~\cite{96}, and affect human health and livelihoods~\cite{97}, etc. As a representative example, in the fake news detection on a social network~\cite{53}~\cite{54}, adversaries have a strong incentive to fool the system by the concealed perturbation in order to avoid detection. In this context, a perturbation could mean modification of the graph structure (inserting or deleting edges in the social graph). 

Traversal perturbing feasible edges are widely exploited in constructing edge-perturbing attacks. Trough the message passing~\cite{89}~\cite{90} driven by end-to-end training, perturbtions will pass to the target along the edges. At the same time, the attacker also designs different constraints to make the perturbations satisfying different indicators. While traversal perturbing, the attackers can observe two most important indicators: attack success rate~\cite{80}~\cite{81}, or the concealment of the selected perturbations~\cite{19}~\cite{20}. Due to the discrete characteristic of the graph, the attacker can always find a concealed perturbation method in an traversal way, making classification result of the target node manipulated. In this way, the defender always lags behind the attacker~\cite{106,55}, so we should never assume that the attackers will not be able to destroy them in the future~\cite{70,72,73,74,75,76}.

Existing defense schemes against edge-perturbing attacks fall into the following two categories: 

\begin{itemize}
    \setlength{\itemsep}{0pt}
    \setlength{\parsep}{0pt}
    \setlength{\parskip}{0pt}
    \item \textit{Adversarial training based schemes:} Adversarial training aims to adapt the model to all possible disturbances. The main idea is to alternatively optimize two competing modules during training~\cite{82,83,85,86}, so as to allow the model adapting to edge perturbations~\cite{77,78,79}. 
    \item \textit{Robust aggregation based schemes:} Robust aggregation aims to design a robust aggregation function~\cite{31,94} of GCNs, and further enables the model to identify and filter the potential perturbations~\cite{87,88,89_5}.
\end{itemize}

However, with the underlying assumption that all edges can be perturbable, malicious adversarial graph may introduce a big challenge to the robustness of GCNs, which can compromise the above two types of schemes (in later sections, we show that the state-of-the-art defense scheme~\cite{38} is vulnerable, 72.8\% nodes under protection are misclassified to the target categories). Specifically, these weaknesses include:

\begin{itemize}

\item  Though adversarial training is very effective in defending against adversarial attacks in computer vision and natural language processing, it is not suitable for graphs, because there is no boundary for malicious disturbances on graphs that hold more complex structures than image or text data. Although some studies~\cite{1}~\cite{2} quantified the upper bound of the number of perturbations, it is still non-intuitive. Compared to the computer vision, the perturbations on images cannot be observed by naked eyes~\cite{98}~\cite{99}, which can significantly reduce the feasible perturbation range. This leads to a strong constraint to the adversarial training in computer vision. However, GCNs suffer from the oversmoothing issues~\cite{91,92,93}, which keeps the properties of adjacent nodes to be consistent. In other words, it is difficult to adapt GCNs to all feasible perturbations while not compromising the model’s utility, attributed to the non-intuitive upper bound of perturbations.

\item  \textit{Robust aggregation} quantifies every step of message passing. A robust aggregation function is a necessary condition to ensure the benign message passing. However, high-epochs end-to-end training will bring huge data flow for the model. Since the perturbations are maliciously designed, and the robust aggregation function is built upon the manual rules, well-designed perturbations can damage the preset rules of the aggregation function with the increase of the epoch.
\end{itemize}

In this paper, by deconstructing GCNs, we first demonstrate the vulnerability of the preconditions of GCNs, i.e., the \textit{edge-reading permission} makes it vulnerable, and further withdraw the \textit{edge-reading permission} of the GCN to defend against such edge-perturbing attack. The motivation of our work comes from the inherent observations to the phenomenon that, in practical scenarios, attackers can always rewire the graph topology by utilizing the edge-reading permission to the graph database, thus misleading the node classification results by the GCNs. As an illustrational example, in financial social networks, users and their social connections can be formed as a graph, and the credit of users in the graph can be predicted by GCNs~\cite{100,101}. For ensuring prediction to be accurate, the graph database must open its access permission to the GCN model. This leaves a big chance for attackers to steal and modify the social relationship between users, and further significantly improve the credit of the target user maliciously. For another example, in the anomaly detection for cloud components, relationships between system components are usually organized by subjective metrics, which is also the necessary input for GCNs~\cite{102,103}. Since organized modeling metrics can be exposed~\cite{102,104,105}, attackers are able to maliciously calculate the cloud components' topology themselves, and further modify metrics between components; thus, can successfully hide the anomalous components. 

The aforementioned examples once again reveal the vulnerability of GCNs. In purpose of designing an effective defense scheme against adversarial attacks, in this paper we demonstrate the preconditions that make GCNs vulnerable. Specifically, we unify the existing manual perturbation based edge-perturbing attacks into an automatic general edge-perturbing attack (G-EPA) model. G-EPA, built upon a surrogate GCN, can reduce a \textit{complete graph} to an optimal attack graph by considering both attack concealment and success rate. It may compromise many state-of-the-art defense schemes~\cite{38}. Since G-EPA is a unified representation of edge-perturbing attacks, we demonstrate the graph vulnerability through the mathematical formula, i.e., GCNs are vulnerable if them directly take edges as inputs.

Following this, we propose an anonymous GCN (AN-GCN) to withdraw the \textit{edge-reading permission} of the model. This eliminate the opportunity for attackers to contact edges information, which is effective for defending against edge-perturbing attacks in realistic scenarios. Corresponding to the above illustrational examples, in financial social networks, AN-GCN can anonymously handle relationships between all users. Thus, from the perspective of attackers, the relationship between users is invisible, so as to eliminate the possibility of edge-perturbing attack. In the anomaly detection for cloud components, even if the attacker calculates the metrics between components according to the \textcolor{red}{exposed method}, the anonymity of AN-GCN itself will \textit{refuse} to accept any edges. This \textcolor{red}{invalidates} all potential malicious perturbations. Specifically, after identifying the causes of GCNs vulnerability, firstly, we state a \textit{node localization theorem} to make the node position computable. For the proof, by regarding the feature change of each node in the training phase as the independent signal vibration, we map the signals of all nodes to the Fourier domain. Thus unify the nodes' signals to a same coordinate system. Finally, we build the \textit{node signal model} to figure out how GCNs deconstructs the initial input to keep the node fixed in its position during training, thus proving the stated node localization theorem.

According to the given node localization theorem, we directly generate the node position by a generator (a fully connected network) to replace the corresponding part in GCN. To ensure generation effective, we improve the pioneering spectral GCN~\cite{22} to a discriminator, which tries to detect nodes with the generated position. Because the discriminator is a self-contained GCN, it can classify the node with generative position accurately once the above two-player game is well played.

In our proposed AN-GCN, the positions of nodes are generated from the noise randomly, while ensuring the high-accuracy classification, that is, nodes are classified while keeping \textit{anonymous} to their positions. Since the edge decides the node position, the anonymity of the position eliminates the necessity of edge-reading for AN-GCN, thus eliminating the possibility of modifying edges, so as to solve the vulnerability of GCNs. Our main contributions include:
\begin{itemize}

    \item We propose an AN-GCN without the necessity of edge-reading while maintaining the utility of classification. This ensures that neither roles (including attackers) can read the edge information, so as to essentially defend against edge-perturbing attack.

    \item We unify the existing edge-perturbing attack methods as a general attack model, and further demonstrate the preconditions for the vulnerability of GCNs, i.e., the edge-reading permission to GCNs makes it vulnerable.
    
    \item We state a node localization theorem to make node position during training computable, which provides a theoretical basis for the proposed AN-GCN.

\end{itemize}

The rest of the paper is organized as follows. In Sec.~\ref{secvul}, we propose G-EPA and demonstrate the vulnerability of GCNs. In Sec.~\ref{secisgcn}, we first state and prove the location theorem, it makes node position computable. We then describe the structure of the proposed AN-GCN, including the generator, discriminator, and training methods. In Sec.~\ref{seceva}, we have comprehensively evaluated the theoretical results and models proposed in this paper.

\section{Related Work}

\subsection{Graph Attacks Against GCNs} In 2018, Dai \textit{et al.}~\cite{1} and Zügner \textit{et al.}~\cite{2} first proposed adversarial attacks on graph structures, after which a large number of graph attack methods were proposed. Specific to the task of node prediction, Chang \textit{et al.}~\cite{3} attacked various kinds of graph embedding models with black-box driven method, Aleksandar \textit{et al.}~\cite{4} provided the first adversarial vulnerability analysis on the widely used family of methods based on random walks, and derived efficient adversarial perturbations that poison the network structure. Wang \textit{et al.}~\cite{5} proposed a threat model to characterize the attack surface of a collective classification method, targeting on adversarial collective classification. Basically, all attack types are based on the perturbed graph structure targeted by this article. Ma \textit{et al.}~\cite{19} studied the black-box attacks on graph neural networks (GNNs) under a novel and realistic constraint: attackers have access to only a subset of nodes in the network. Xi \textit{et al.}~\cite{71} proposed the first backdoor attack on GNNs, which significantly increases the mis-classification rate of common GNNs on real-world data.

\subsection{Defenses for GCNs} Tang \textit{et al.}~\cite{6} investigated a novel problem of improving the robustness of GNNs against poisoning attacks by exploring a clean graph and created supervised knowledge to train the ability to detect adversarial edges so that the robustness of GNNs is elevated. Jin \textit{et al.}~\cite{7} used the new operator in replacement of the classical Laplacian to construct an architecture with improved spectral robustness, expressivity, and interpretability. Zügner \textit{et al.}~\cite{8} proposed the first method for certifiable (non-)robustness of graph convolutional networks with respect to perturbations of the node attributes. Zhu \textit{et al.}~\cite{38} proposed RGCN to automatically absorb the effects of adversarial changes in the variances of the Gaussian distributions. Some defense methods use the generation to enhance the robustness of the model. Deng \textit{et al.}~\cite{9} presented batch virtual adversarial training (BVAT), a novel regularization method for GCNs. By feeding the model with perturbing embeddings, the robustness of the model is enhanced. But this method trains a full-stack robust model for the encoder and decoder at the same time without discussing the essence of the vulnerability of GCN. Wang \textit{et al.}~\cite{10} investigated latent vulnerabilities in every layer of GNNs and proposed corresponding strategies including dual-stage aggregation and bottleneck perceptron. To cope with the scarcity of training data, they proposed an adversarial contrastive learning method to train GNN in a conditional GAN manner by leveraging high-level graph representation. But from a certain point of view, they still used the method based on node perturbation for adversarial training. This method is essentially a kind of 
``perturbation" learning and uses adversarial training to adapt the model to various custom perturbations. Feng~\cite{58} proposed the Graph Adversarial Training (GAD), which takes the impact from connected examples into account when learning to construct and resist perturbations. They also introduces adversarial attack on the standard classification to the graph. 

% \textbf{Graph GAN without considering attack and defense}. Wang~\cite{11} combined two methods of graph representation learning as generators and discriminators, respectively, to improve the accuracy of both in adversarial training. However, this method does not discuss the potential vulnerability of the graph structure, nor does it attempt to accurately perturb the final classification, and cannot be directly applied to the graph defense method. Ding~\cite{12}'s perspective is extended to the regional structure of the entire graph, but the task goal is still to obtain an accurate graph representation, and the generated fake samples cannot match the various perturbation that is carefully designed for the model vulnerabilities.

\section{The vulnerability of GCNs}\label{secvul}

In this section, we first give a general attack model by re-examining the detailed phase of graph convolution, and then prove that the vulnerability of the graph is caused by the edge-reading permission. Specifically, in Sec.~\ref{sec3_1}, we give the mathematical expression formula of the encoding and decoding in GCN. In Sec.~\ref{sec3_2}, we formulate the general attack model as G-EPA and demonstrate the vulnerability of GCNs, so as to give the motivation of our proposed defense scheme. In Sec.~\ref{sec3_3}, we give a case study to show how to attack a specific model in details.

\subsection{Graph Convolution Encoding, Decoding and Training}\label{sec3_1}

Let $\mathcal{G}=\left(f,\mathcal{E} \right)$ be a graph, where $f=\left( f ( 1 ), \cdots ,f ( N )  \right)^{\top}$ is a set of features of $N$ nodes, while $f(i)$ is the feature vector of node $i$, and $\mathcal{E}$ is the set of edges, which also decide the positions of all nodes. The adjacent matrix $A\in\mathbb{R}^{N \times N}$ and degree matrix $D\in\mathbb{R}^{N \times N}$ can be calculated according to $\mathcal{E}$. The laplacian matrix $\Delta$ of $\mathcal{G}$ can be obtained according to $D$ and $A$. The one-hot encoding of node categories on $\mathcal{G}$ is $\mathcal{Y}$. Since $A$ is calculated by $\mathcal{E}$,  the graph structure (simultaneously, node positions) can be represented by $A$. The phase of the training model on labeled graph is as follows:

\textit{Encoding. }Encode is a node feature aggregation~\cite{34} process according to $\mathcal{E}$. The graph convolution model maps $f$ into the embedding space $f^e=\left(  f^e ( 1 ), \cdots, f^e ( N ) \right)^{\top}$ through trainable parameter $\theta^{E}$, while $f^e(i)$ denotes the embedding feature vector of node $i$. Let $\mathrm{ENC}(\cdot)$ denotes the encoder, the encoding process is denoted as:
\begin{sequation}
\label{eq1}
f ^ { e } = \mathrm{ENC}\left(f;A, \theta^{E}  \right),
\end{sequation}
Encode can represent the discrete graph structure into a continuous embedded space~\cite{25}, which can learn to represent graph nodes, edges or subgraphs in low-dimensional vectors~\cite{30}. Many studies~\cite{26,27,28,29} visualized the encoded graph to evaluate the effectiveness of model.

\textit{Decoding. }Then model decodes $f^e$ to the one-hot label space $\mathcal{L} \sim \mathbb{R}^{\iota \times 1}$ through the decoder $\mathrm{DEC}$ (In general, it is GCN layers) with the trainable parameter $\theta^{D}$, $\iota$ is the number of categories. The output of the decoder provides gradient descent direction for GCN model training (usually, the distinction between encoder and decoder is flexible, for example, in multi-layer GCN, the first few layers can be regarded as encoders, and the remaining layers can be regarded as decoders). Almost all high performance Graph Neural Networks models (e.g., GraphSAGE~\cite{31}, GAT~\cite{32}, spectral-GCN~\cite{14}, etc.) set the encoder and decoder to the same-design layer, that is, both encoding and decoding involve the aggregation of graph structures $\mathcal{E}$, so the decoder parameters still contain $\mathcal{E}$. The decoding process is denoted as:

\begin{sequation}
\label{eqaag}
\mathcal{Y} = \mathrm{DEC} \left(f^e;A, \theta^{E}  \right),
\end{sequation}
The decoder calculates output according to $A$ and $f^e$.

\textit{Training. }According to the labeled nodes, the parameters of encoder and decoder are trained by end-to-end training, thus realizing the accurate node prediction task. In the training phase, the regularization $\mathrm{Reg}\left(\cdot\right)$ ($l1-norm$~\cite{24}, et al.) will prevent the model from over fitting. Then the training phase can be expressed as:
\begin{small}
\begin{multline}
\label{eq3}
\mathop{\arg\min} \limits_{\theta^{GCN}}  \sum_{}^{} \Big\{ \mathrm{Loss}\Big[\mathcal{Y},\\
\mathrm{DEC}(\mathrm{ENC}(f;A, \theta^{E});A, \theta^{D} )\Big]+\mathrm{Reg}(\theta^{GCN}) \Big\},
\end{multline}
\end{small}
where $\mathrm{Loss}$ is the loss function and $\theta^{GCN}$ is the total parameter set, i.e., $\theta^{E}$ and $\theta^{D}$. Equation~\ref{eq3} is the unified form of current graph convolution training methods, and it is computationally universal (Graph neural networks with sufficient parameters are always computationally universal~\cite{21}). The encoding and decoding process $\mathrm{DEC}(\mathrm{ENC}(\cdot))$ is denoted as $\mathbf{GCN}(\cdot)$.

\subsection{General Edge-Perturbing Attacks} \label{sec3_2}
\begin{figure*}
\centering
\includegraphics[width=0.95\textwidth]{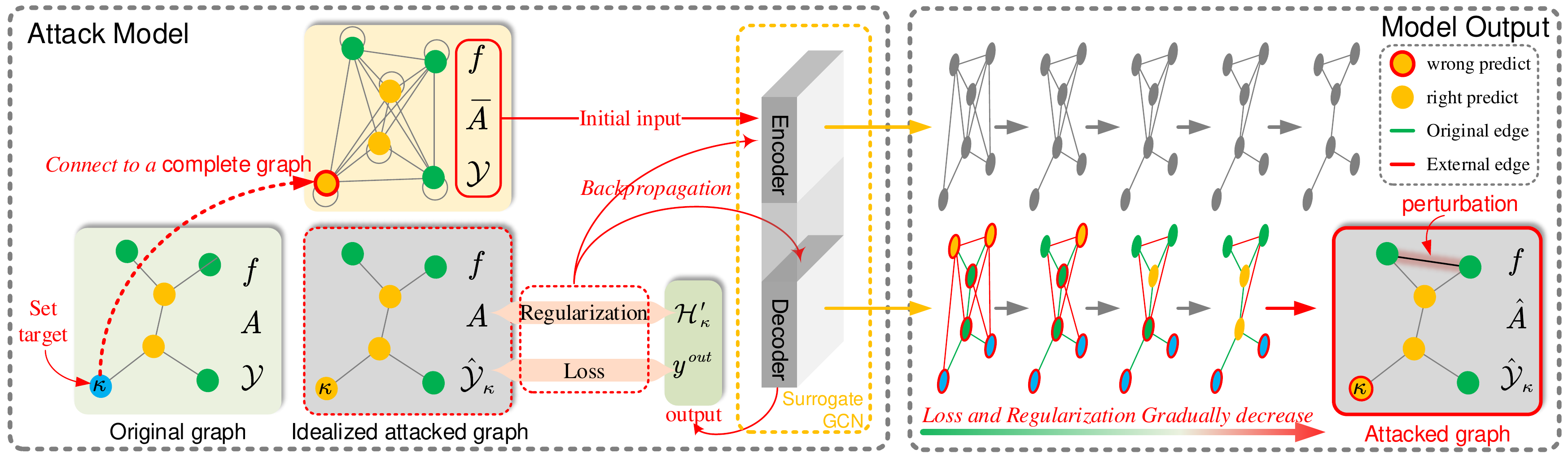}
\caption{The schematic illustration of G-EPA. G-EPA reduces a \textit{complete graph} to a concealed and effective attacked graph. G-EPA iterates while training until the best attack scheme is found.
} \label{figattack}
% \vspace{-0.5cm} 
\end{figure*}

Manual constraints are widely employed in constructing exist edge-perturbing attack. In this section, we propose G-EPA which integrates automatic constraints, thus to realize the concealed attack by automatically traverses all disturbances through end-to-end training. Because the current edge-perturbing attack is also targeted on obtaining the best perturbation scheme, G-EPA can unify the existing edge-perturbing attacks, thus be a general attack model.

We first formulate the aim of the attack in Sec.~\ref{secaim}, then give the specific attack method in Sec.~\ref{secAfind}.

\subsubsection{The aim of the attack}\label{secaim}

Given a target graph $\mathcal{G}_{\mathcal{T}}$ and the target GCN which well-trained on $\mathcal{G}_{\mathcal{T}}$, we denote the parameter of it as the $\theta^{G}_{\mathcal{T}}$. The graph attack perturbs the graph connection, and makes classification of the target node changed by traversing the perturbations on $\mathcal{G}_{\mathcal{T}}$, to further obtain $\mathcal{G}_{victim}=(f,\mathcal{E}_{victim})$. Therfore, the attack is to find a modified graph adjacency matrix $\hat{A}$ corresponding to $\mathcal{E}_{victim}$, let
\begin{sequation}
\label{eqaaa}
\hat{\mathcal{Y}}_{\kappa} = \mathop{\mathbf{GCN}}\left(f; \hat{A}, \theta^{G}_{\mathcal{T}}\right),
\end{sequation}
where $\kappa$ denotes the node of modified category, and $\hat{\mathcal{Y}}_{\kappa}$ denotes the distribution of the modified node category set in $\mathcal{L}$. That is, in $\hat{\mathcal{Y}}_{\kappa}$, the category of $\kappa$ is manipulted as the target category, meanwhile making other nodes keep original categories (whether single-target and multi-target). From the perspective of model optimization optimization~\cite{35}, it could be either evasion attack~\cite{19,20} or poisoning attack~\cite{16,17,18}. Once $\hat{A}$ is found, $\mathcal{E}_{victim}$ can be calculated by $\hat{A}$, to realize the effective attack.

The aforementioned edge-perturbing attack process can be formulated as an $\hat{A}$-finding game: 
\begin{sequation}
\label{eq10}
\hat{A} = \mathop{\mathbf{ATTACK}}\left(\theta^{E}_{\mathcal{T}}, \theta^{D}_{\mathcal{T}}, \mathcal{E}\right),
\end{sequation}
where $\theta^E_{\mathcal{T}}$ and $\theta^D_{\mathcal{T}}$ are the parameters of the encoder and decoder of the target GCN, respectively. In other words, the attacker perturbs $\theta^{E}_{\mathcal{T}}$, $\theta^{D}_{\mathcal{T}}$ and $\mathcal{E}$, thus finds $\hat{A}$ to realize the attack. 

\subsubsection{$\hat{A}$ obtaining method}\label{secAfind}

In this section, we will give the method of obtaining $\hat{A}$.

Equation~\eqref{eqaaa} can be expressed as:
\begin{sequation}
\label{eqaac}
    \hat{\mathcal{Y}}_{\kappa} = \mathop{\mathrm{DEC}} \left[\mathop{\mathrm{ENC}}(f;\hat{A}, \theta^{E}_{\mathcal{T}});\hat{A}, \theta^{D}_{\mathcal{T}}\right],
\end{sequation}
once the attacker successfully obtains $\hat{A}$ in Eq.~\eqref{eqaac}, the attack will be completed, furthermore, it takes advantage of the vulnerability of GCN. 

Next, we will give the method to make the model automatically find the best perturbation scheme, i,e., $\hat{A}$.

Before giving the attack steps, we first give the preparatory work: 

\begin{itemize}

\item Connecting target node $\kappa$ to a \textit{complete graph} $\bar{\mathcal{G}}$, where each node has closed loops, the diagonal matrix and adjacency matrix of $\bar{\mathcal{G}}$ are $\bar{D}$ and $\bar{A}$, respectively. The node feature sets of $\mathcal{G}$ and $\bar{\mathcal{G}}$ are the same, and the node category set of $\bar{\mathcal{G}}$ is the modified set, i.e., $\hat{\mathcal{Y}}_{\kappa}$.
\end{itemize}

In order to find $\hat{A}$ in Eq. \eqref{eqaac}, the attack can be constructed as the following steps:

\begin{itemize}

\item Introduce a \textit{surrogate} GCN which has the same layer structure and parameters with the target GCN. In the \textit{surrogate} GCN, $\theta^{E}_{\mathcal{T}}$ and $\theta^{D}_{\mathcal{T}}$ are frozen.

\item Set a trainable parameter matrix $\mathcal{H}$ in \textit{surrogate} GCN.

\item Take $\bar{\mathcal{G}}$ as the input of \textit{surrogate} GCN.

\item Forward propagation. In training epoches, $\hat{A}$ is dynamically changed according to $\mathcal{H}$. Specifically, let $\mathcal{H}+\mathcal{H}^{\top}$ act on $\bar{A}$ by Hadamard Product. Other calculation methods remain unchanged with target GCN.

\item Back propagation. Calculate the loss of the output and $\hat{\mathcal{Y}}_{\kappa}$, and update $\mathcal{H}$ according to the loss (reminding that $\theta^{E}_{\mathcal{T}}$ and $\theta^{D}_{\mathcal{T}}$ is frozen).
\end{itemize}
The schematic illustration of the attack is shown as Fig.~\ref{figattack}.

In the attack process, Hadamard product makes trainable matrix modify $\bar{A}$ element by element, and let $\mathcal{H}$ plus $\mathcal{H}^{\top}$ ensure the symmetry of the modified matrix. Let $\mathcal{H}^{\prime}=\mathcal{H}+\mathcal{H}^{\top}$, the adjacency matrix after attack 
\begin{sequation}\label{eqA}
    \hat{A} = \mathcal{H}^{\prime} \odot\bar{A},
\end{sequation}
where $\odot$ represents Hadamard Product. Since $\bar{A}$ is a matrix in which all elements are 1 ($\bar{\mathcal{G}}$ is a full connection graph with closed loop), so $\hat{A} = \mathcal{H}^{\prime}$.

In other words, the attack uses a surrogate model to make $\hat{A}$ in Eq. \eqref{eqaac} dynamically changed. Equation~\eqref{eqaac} can be expressed as:
\begin{sequation}
\label{eqaad}
\hat{\mathcal{Y}}_{\kappa} = \mathop{\mathrm{DEC}}_{surrogate}\left[\mathop{\mathrm{ENC}}_{surrogate}(f;\mathcal{H}^{\prime}, \theta^{E}_{\mathcal{T}} );\mathcal{H}^{\prime}, \theta^{D}_{\mathcal{T}}\right],
\end{sequation}
where $\mathop{\mathrm{ENC}}\limits_{surrogate}$ and $\mathop{\mathrm{DEC}}\limits_{surrogate}$ are the encoder and decoder of the surrogate GCN, respectively.

Next, considering that the attacker will not perturb the edge directly connected to $\kappa$ for realized concealed attack, the $\mathcal{H}^{\prime}$ will not directly perturb $\kappa^{\text{th}}$ line of $\bar{A}$, that is, $\mathcal{H}^{\prime}$ in Eq. \eqref{eqaad} can be replaced with $\mathop{\mathcal{H}^{\prime}}\limits_{\kappa}= ( \mathcal{H}^{\prime}_{1}, \cdots , A_{\kappa} , \cdots , \mathcal{H}^{\prime}_{N})^{\top}$, where $\mathcal{H}^{\prime}_{i}$ denotes the $i^{\text{th}}$ row of $\mathcal{H}^{\prime}$. 

Furthermore, in order to ensure the global concealment, attackers need to minimize the amount of perturbation. Specifically, it can be realized by adding a regularization, which can be denoted as $\mathrm{Reg}(\mathop{\mathcal{H}^{\prime}}\limits_{\kappa}) = \mathrm{Sum}(A - \mathop{\mathcal{H}^{\prime}}\limits_{\kappa} \odot\bar{A})$, where $\mathrm{Sum}$ is the sum of all elements of matrix. The regularization can quantify the perturbation degree of $\mathop{\mathcal{H}^{\prime}}\limits_{\kappa}$ to $\bar{A}$.

Hence, the attacker can realize the attack by obtain $\mathop{\mathcal{H}^{\prime}}\limits_{\kappa}$ according to Eq.~\eqref{eqattack1}.
\begin{small}
\begin{multline}
\label{eqattack1}
  \mathop{\arg\min} \limits_{\mathop{\mathcal{H}^{\prime}}\limits_{\kappa}} \sum_{}^{} \Big\{ \mathrm{Loss} \Big[ \hat{\mathcal{Y}}_{\kappa}, \\
 \mathop{\mathrm{DEC}}_{surrogate}( \mathop{\mathrm{ENC}}_{surrogate}(f;\mathop{\mathcal{H}^{\prime}}\limits_{\kappa}, \theta^{E}_{\mathcal{T}} );\mathop{\mathcal{H}^{\prime}}\limits_{\kappa}, \theta^{D}_{\mathcal{T}})\Big]  + \mathrm{Reg}(\mathop{\mathcal{H}^{\prime}}\limits_{\kappa})\Big\}.
\end{multline}
\end{small}
Finally, $\hat{A}$ can be calculated by Eq.~\eqref{eqA}. 

Equation~\eqref{eqattack1} is the mathematical form of the ERA. In Sec.~\ref{seceva}, we will evaluate the performance of Eq.~\eqref{eqattack1} by implementing the attack.

\subsection{The Statement of GCNs Vulnerability}\label{secvulstate}

As can be seen, Eqs.~\eqref{eqattack1} and~\eqref{eq3} are essentially the same, since Eq.~\eqref{eq3} is the unified training method of GCN, hence, Eq.~\eqref{eqattack1} is computationally universal. That is, \textit{if the attacker wants to modify the node category set, it can be realized by another GCN training task}. Since Eq.~\eqref{eqaaa} is the general mathematical formula for attacks and Eq.~\eqref{eqattack1} is derived from Eq.~\eqref{eqaaa}, so we conclude that attackers can precisely control the GCN output by modifying graph edges. In other words, no matter what the values of $\theta^{E}_{\mathcal{T}}$ and $\theta^{D}_{\mathcal{T}}$ are, the attack can be realized by manipulating $\mathcal{E}$. In other words, attackers can bypass $\theta^{E}_{\mathcal{T}}$ and $\theta^{D}_{\mathcal{T}}$ by manipulating $\mathcal{E}$.

In order to bypass the encoder and decoder by manipulating $\mathcal{E}$, Eq.~\eqref{eq10} becomes:
\begin{sequation}
\label{eq15}
\hat{A} = \mathop{\mathbf{ATTACK}}_{essence}\left(\mathcal{E}\right).
\end{sequation}
Equation~\eqref{eq15} shows the essence of the attack target is the exposed graph structure, that is, edges of the graph. As long as the GCN model directly receives the graph edges as the input, the feasible perturbation can be found through Eq.~\eqref{eqattack1} to construct edge-perturbing attacks. In other words, the vulnerability of the GCN is caused by the edge-reading permission.

\begin{figure}[htb]
\centering
\includegraphics[width=7.2cm]{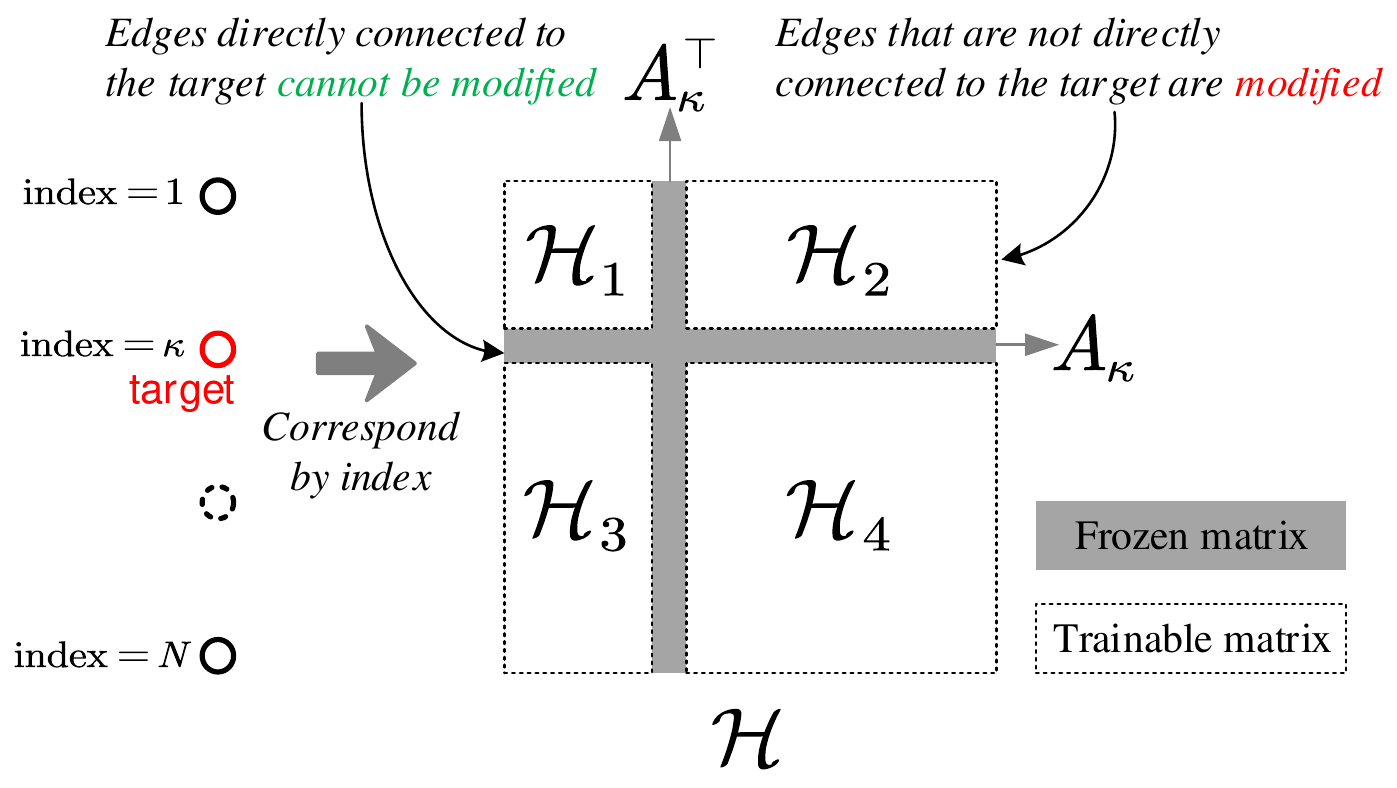}
\caption{Trainable parameter matrix $\mathcal{H}$ under single node attack.} \label{fig1}
\end{figure}
\subsection{Case Study}\label{sec3_3}

In this section, we elaborate on how to attack a specific GCN. We use the semi-GCN~\cite{14} as the specific model to obtain $\mathcal{G}_{victim}$, the obtaining method can be expressed as
\begin{small}
\begin{multline}
\label{eqattacksemi}
  \mathop{\arg\min} \limits_{\mathop{\mathcal{H}^{\prime}}\limits_{\kappa}} \sum_{}^{}\Big\{ \mathrm{Loss} \Big[\hat{\mathcal{Y}}_{\kappa},\\
  \mathrm{Softmax}(\mathop{\mathcal{H}^{\prime}}\limits_{\kappa}\mathrm{Relu}(\mathop{\mathcal{H}^{\prime}}\limits_{\kappa} f W_1 )W_2 )\Big] + \mathrm{Reg}(\mathop{\mathcal{H}^{\prime}}\limits_{\kappa})\Big\},
\end{multline}
\end{small}
where $W_1$ and $W_2$ are well-trained parameters of target GCN in different layers. Next, we will give the details of the the attack, including trainable parameter matrix, layer weight constraint, layer weight initializer and regularizer of the model. Considering that the existing attack scenarios include the single target scenario and the multiple targets scenario, we will describe them separately.

\subsubsection{Single target}

\textit{Trainable parameter matrix.} In order to ensure $\mathop{\mathcal{H}^{\prime}}\limits_{\kappa}= (  \mathcal{H}^{\prime}_{1}, \cdots , A_{\kappa} , \cdots , \mathcal{H}^{\prime}_{N})^{\top}$, we set the trainable parameters $\mathcal{H}$ to 4 sub-matrices: $\mathcal{H}_{1} \in \mathbb{R}^{(\kappa-1)\times(\kappa-1)}$, $\mathcal{H}_{2} \in \mathbb{R}^{(\kappa-1)\times(N-\kappa)}$, $\mathcal{H}_{3} \in \mathbb{R}^{(N-\kappa)\times(\kappa-1)}$, $\mathcal{H}_{4} \in\mathbb{R}^{(N-\kappa)\times(N-\kappa)}$, and they are placed in $\mathcal{H}$ as shown in the Fig.~\ref{fig1}.

\textit{Layer weight constraint.} In order to make the modified $\hat{A}$ only have two elements: 0 and 1, we set up layer weight constraints term:
\begin{sequation}
    \mathcal{H}_{i,j}=\left\{ \begin{array} { l } 0 \ \stt \  A_{i,j} \geq \mathrm{O}_{i,j} \\
    1\  \stt \  A_{i,j} < \mathrm{O}_{i,j} \end{array} \right.,
\end{sequation}
where $\mathcal{H}_{i,j}$ denotes the element in $\mathcal{H}$ with indices as $i$ and $j$. $\mathrm{O}$ denotes parameters before layer weight constraint.

\textit{Layer weight initializer. } Each trainable sub-matrix in $\mathcal{H}$ is initialized as elements, which is the same as the corresponding indices of $A$:
\begin{equation}
    \mathcal{H}^{(0)}_{i,j} = A^{(0)}_{i,j}.
\end{equation}

\subsubsection{Multiple targets}

In the case of multiple targets, let $\mathcal{K} = \{\kappa_1,\ldots,\kappa_N\}$ denote the target node set. 

\textit{Trainable parameter matrix.} The model flexibility of multi-target attack will be limited by setting too many sub-matrices, so all the elements of $\mathcal{H}$ can be trained, and ensuring the concealment of attacks in regularizer.

\textit{Regularizer.} To minimize the number of direct perturbations to the target node, we need to minimize the number of modifications to the $\kappa$-th row of $\bar{A}$. The regularizer for multi-target attack is designed as:
\begin{sequation}
    \mathrm{Sum}\left(A - \mathop{\mathcal{H}^{\prime}}\limits_{\kappa} \odot\bar{A}\right) + \vartheta \sum_{\kappa \in \mathcal{K}}  \mathrm{Sum}\left( h_{\kappa}  \right) ,
\end{sequation}
where $h_{\kappa}$ denotes the $\kappa$-th row of $\mathcal{H}$ and $\vartheta$ denotes the custom coefficient.

\section{AN-GCN: Keep Node Classification Accurate and Anonymous}\label{secisgcn}

In Sec.~\ref{secvul}, the vulnerability of GCN is demonstrated as the edge-reading permission of GCN. In this section, we will withdraw the edge read permission of GCN. We integrate a generator and a self-contained GCN to ensure that all node positions are generated from the noise, to ensure that the classification is only according to the node index and feature, thus makes GCN no longer relying on the information of the edge for classification.

\subsection{How GCNs Locating Nodes}\label{secloc}

Because AN-GCN is capable of generating node positions, in this section, we first state a node localization theorem as the theoretical basis of AN-GCN.

The messages are constantly passing while the GCN training, which also can regarded as a kind of signal transmission, meanwhile, the features of the node are constantly changing. We regard it as an \textit{independent signal vibration}. In this section, we will build a mathematical model of the node signal vibration to figure out how GCN analyzes the given graph structure to locate nodes. In Sec.~\ref{secmath}, the basic mathematical model for node signal is given. In Sec.~\ref{secstartup}, the node signal in the Fourier domain driven by GCN is described, whose processes includes two parts: $\mathit{1}$) The initial position of the node signal in the Fourier domain is given by blocking the transmission of graph signal; $\mathit{2}$) The model of node signal change in the GCN process is given by the introduction of trainable parameters of signal transmission on graphs. Specifically, we first regard the change of each node features as the signal changing with time (in training, it's the form of epoch) $ t $, and then map all the node signals to the same orthogonal basis through Fourier transform. Furthermore, we give the node signal model in GCN training, which is used to find the method of GCN locating specific nodes. The method of building the node signal model is shown in the Fig.~\ref{figfit}. 

\notation~Let $\mathcal{G}=\left(f,\mathcal{E} \right)$ be a graph, where $f$ is set of features of $N$ nodes, while $f(i)$ is the feature of node $i$ and $\mathcal{E}$ denotes the set of edges. An essential operator in spectral graph analysis is the graph Laplacian, whose combinatorial definition is $\Delta=D-A$ where $D\sim\mathbb{R}^{N \times N}$ if the degree matrix and $A\sim\mathbb{R}^{N \times N}$ is the adjacent matrix (Both $D$ and $A$ can be calculated from $\mathcal{E}$). Let the Laplacian matrix of $\mathcal{G}$ is $\Delta$ where eigenvalues are $\lambda_1, \ldots, \lambda_N$, and the corresponding eigen matrix of $\Delta$ is {\small $U=\left( \begin{array} { c c c } u _ { 1 } ( 1 ) & \cdots & u _ { N } ( 1 ) \\ \vdots & \ddots & \vdots \\ u _ { 1 } ( N ) & \cdots & u _ { N } ( N ) \end{array} \right)$}, $u_l=\left(u_l(1),\cdots, u_l(N)  \right)^{\top}$ is the $l$-th eigenvector, $u(l)=\{u_1(l),\ldots,u_N(l)$\}  is the row vector consisting of the values of all eigenvectors at position $l$. For convenience, ``node with the index $n$" is denotes by ``node $n$". $j$ is the imaginary unit.

\subsubsection{Node localization theorem}
Before making a mathematical proof, we first claim that:

\begin{theorem}\label{claim}
% \vspace{0.2cm}
Giving a graph $\mathcal{G}$ to be learned by a GCN, the GCN locates the node $\alpha$ according to $u(\alpha)$.
% \vspace{0.2cm}
\end{theorem}

The proof of Theorem~\ref{claim} is provided in the following Sec.~\ref{secmath} and Sec.~\ref{secstartup}, respectively.

\subsubsection{Node signal model}\label{secmath}

\begin{figure*}[htb]
\centering
\includegraphics[width=0.8\textwidth]{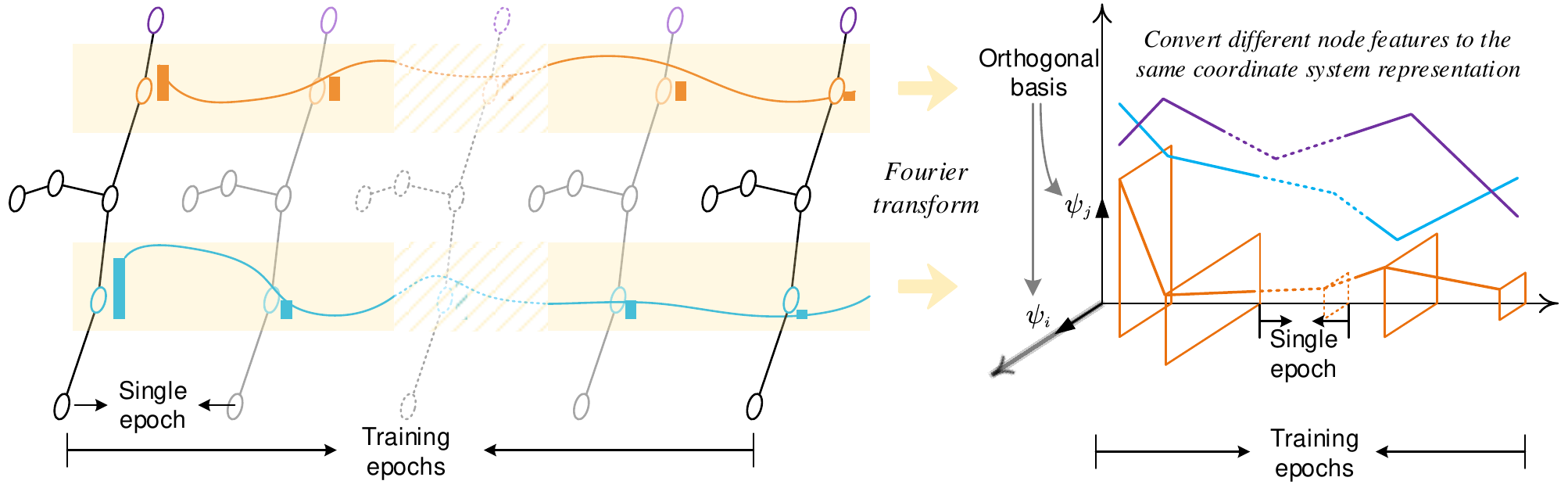}
\caption{The method of building node signal model. We regard the node feature change with the training epoch as the independent signal vibration, and transform all the signals to a unified orthogonal basis by Fourier transform.} \label{figfit}
\vspace{-0.0cm}
\end{figure*}

Firstly, we give the Fourier domain coordinates of the signal of a single node at a specific frequency:

\begin{proposition}
Giving a frequency $\nu$, the coordinates in the Fourier domain of signal of node $\alpha$ are:

\begin{equation}\label{pro1}
\hat{f}_{\alpha}[\nu] = \sum_{t=0}^{E-1}\bar{\theta_{t}}\left(\sum_{i=1}^{N} c_i e^{\lambda_{i}t} u_i(\alpha) \right) e^{-j\frac{2\pi}{E}t\nu},\end{equation}
where $\bar{\theta}_{ t }$ and $c_i$ are constants, $E$ is the number of training epoch.
% \vspace{0.2cm}
\end{proposition}

\begin{proof}
The signal variation of a single node with $E$ epochs is presented as $f[0]$, $f[1]$,...,$f[E-1]$. According to the Discrete Fourier Transform (DFT) theory, given a frequency $\nu$, the fourier transform for the signal $f$ of node $\alpha$ (before introducing trainable parameters) is
\begin{equation}
    \hat{f}_{\alpha}[\nu]=\sum^{E-1}_{ t =0}f_{\alpha}[ t ]e^{-\frac{2\pi}{E}\nu t  j},
\end{equation}
In the Fourier domain, according to the aggregation theory on graph~\cite{31}, the signal in Fourier domain is dynamically vibrating caused by trainable parameters, so here we introduce a constant matrix representing vibration frequencies $\bar{\theta}\sim\mathbb{R}^{E}$ where $\bar{\theta}_i$ is the value of $\bar{\theta}$ at the $i^{\text{th}}$ node, by act $\bar{\theta}$ on all frequency $\hat{f}[\nu], \nu=0:E-1$ at every epoch. The signal fitted by $\bar{\theta}$ will be obtained, as shown in the Fig.~\ref{figFlow}, specifically.

\begin{figure}[htb]
\centering
\includegraphics[width=0.4\textwidth]{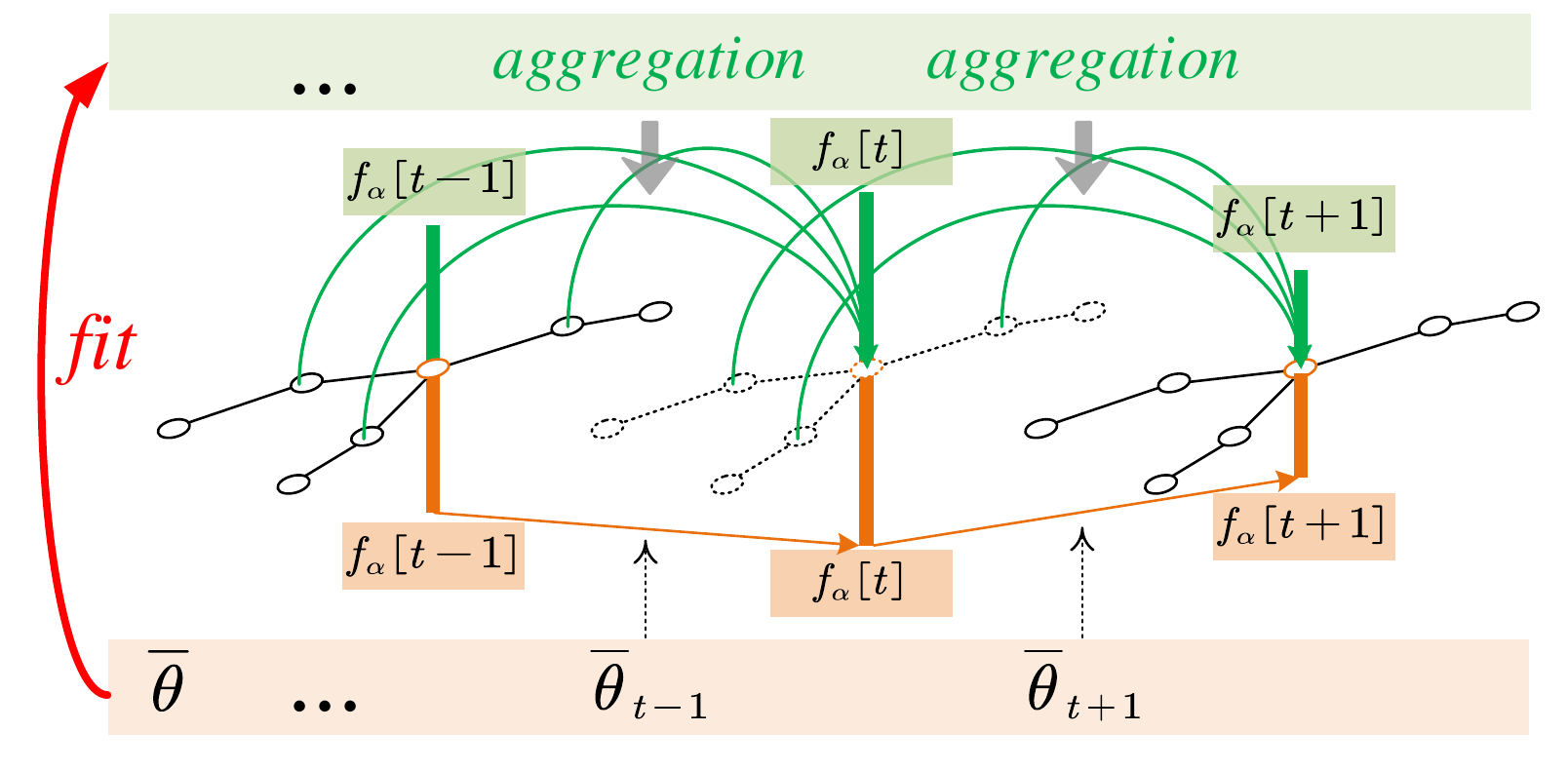}
\caption{Using $\bar{\theta}$ to fit graph quantify aggregation.} \label{figFlow}
\vspace{-0.0cm}
\end{figure}

Hence the Fourier domain graph signal (after introducing trainable parameters) is
\begin{equation}\label{proo2}
        \hat{f}_{\alpha}[\nu]=\sum^{E-1}_{ t =0}\bar{\theta}_i f_{\alpha}[ t ]e^{-\frac{2\pi}{E}\nu t  j}.
\end{equation}
Furthermore, according to the model of signal transmission on the graph~\cite{56}, the transmission of signal $f$ on $\mathcal{G}$ is proportional to $\Delta$ acting on $f$ with time $t$. As the end-to-end training drives the signal transmission on the graph, by regarding $t$ as the end-to-end training epoch $t$, so 
\begin{equation}\label{dfdt}
    \frac{df}{dt} = -k\Delta f,
\end{equation}
where $k$ is a constant. The first-order matrix ordinary differential equation~\cite{65,66} is
\begin{equation}\label{mod}
    \dot{\chi}(t) = \mathscr{L}(t)\chi(t),
\end{equation}
where $ \dot{\chi}$ is the vector of first derivatives of these functions, and $\mathscr{L}(t)$ an $N\times N$ matrix of coefficients. In the case where $\mathscr{L}(\cdot)$ is constant and has $N$ linearly independent eigenvectors, the general solution of Eq.~\eqref{mod} is
\begin{equation}
\chi(t) = c_1 e^{\dot{\lambda}_1 t}\rho_1 + c_2 e^{\dot{\lambda}_2 t}\rho_2 + \ldots + c_N e^{\dot{\lambda}_N t}\rho_N,
\end{equation}
where $\dot{\lambda}_1, \ldots, \dot{\lambda}_N$ are the eigenvalues of $\mathscr{L}$, $\rho_1, \ldots, \rho_N$, as $\Delta\sim\mathbb{R}^{N\times N}$ has $N$ linearly independent eigenvectors eimilarly, so Eq.~\eqref{dfdt} is homogeneous to Eq.~\eqref{mod} that has the general solution, which is given as
\begin{equation}\label{ft}
    f_{\alpha}[t] = \sum_{i=1}^{N} c_ie^{\lambda_i t}u_i(\alpha),
\end{equation}
Eq.~\eqref{ft} can be regard as the global change model of graph signal transmission, hence Eq.~\eqref{pro1} can be acquired by substituting  Eq.~\eqref{ft} into Eq.~\eqref{proo2}.
\end{proof}

Then, for node $\alpha$, we integrate the signals of all frequencies in the Fourier domain to obtain the original signal in the orthogonal basis representation, which is stated in the following proposition.

\begin{proposition}\label{proposition1}
The signal of node $\alpha$ in epoch $t$ is:
\begin{multline}\label{pro2}
    f_{\alpha}[t] = \sum_{\nu=0}^{E-1}\frac{2}{E} \left| 
    \hat{f}_{\alpha}[\nu]
    \right|
    \cos \Big[ \left(  { 2 \pi\nu} / { E \epsilon } \right)  t  \epsilon + \\
    \arg ( \hat{f}_{\alpha}[\nu] ) \Big] e^{j\frac{2\pi}{E}t\nu} 
    ,
\end{multline}
where $\arg$ is the argument of a complex number, $\epsilon$ denotes the epoch interval, here we take it as a minimal value, that is, $\epsilon \to 0$.
\end{proposition}

\begin{proof}
According to the Discrete Fourier Transform (DFT) theory, the inverse transform of $ \hat{f}_{\alpha}[\nu]$ is 
\begin{equation}\label{inverse}
    f_{\alpha}[t] = \frac{1}{E}\sum_{\nu=0}^{E-1} \hat{f}_{\alpha}[\nu] e^{j\frac{2\pi}{E}\nu t},
\end{equation}
i.e., the inverse matrix is $\frac{1}{N}$ times the complex conjugate of the original (symmetric) matrix. Note that the $\hat{f}_{\alpha}[\nu]$ coefficients are complex. We can assume that the graph signal $f[t]$ values are real (this is the simplest case, there are situations in which two inputs, at each $t$, are treated as a complex pair, since they are outputs from $0^{\circ}$ and $90^{\circ}$ demodulators). In the process of taking the inverse transform for the term $\hat{f}_{\alpha}[\nu]$ and $\hat{f}_{\theta}[E - \nu]$ (the spectrum is symmetrical about $\frac{E}{2}$~\cite{51}) combine to produce 2 frequency components, only one of which is considered to be valid. Hence from the Eq.~\eqref{inverse}, the contribution to $f_{\alpha}[t]$ of $\hat{f}_{\alpha}[\nu]$ and $\hat{f}_{\alpha}[E - \nu]$ is:
\begin{equation}\label{contri1}
    \mathcal{C}^{\nu}_{\alpha}[ t ] = \frac{1}{E}\left(\hat{f}_{\alpha}[\nu]e^{j\frac{2\pi}{E}\nu t} + \hat{f}_{\alpha}[E - \nu]e^{j\frac{2\pi}{E}(E-\nu) t} \right).
\end{equation}
Consider that all graph signals $f[ t ]$ are real, so
\begin{equation}
    \hat{f}_{\alpha}[E - \nu] = \sum_{t=0}^{E-1}f[t]e^{-j\frac{2\pi}{E}(E-\nu)t},
\end{equation}
and according to the Euler's formula~\cite{67}, \begin{equation}\label{euler}
    e^{-j 2\pi t} = \frac{1}{\cos{(2\pi t)} + j\sin{(2\pi t)}} = 1,\stt\  t \in \bm{\mathrm{N}_{+}},
\end{equation}
where $\bm{\mathrm{N}_{+}}$ is the set of all positive integers. We have
\begin{multline}
\label{eqe1}
    e ^ { - j \frac { 2 \pi } { E } ( E - \nu ) t } = \underbrace { e ^ { - j 2 \pi t } } _ { 1 \text { for all } t } e ^ { + j \frac { 2 \pi \nu } { E } t } = e ^ { + j \frac { 2 \pi } { E } \nu t },\\
    \ie \ \hat{f}_{\alpha}[E - \nu] =  \hat{f}_{\alpha}^{*}[\nu],
\end{multline}
where $\hat{f}_{\theta}^{*}[\nu]$ is the complex conjugate, substituting Eq.~\eqref{eqe1} into Eq.~\eqref{contri1}, as $e^{j 2\pi t}=1$ (refer to Eq.~\eqref{euler}) we have
% \begin{align}
%      f_{(\nu)}[ t ] & = \frac{1}{E}\left(\hat{f}_{\alpha}[\nu] e^{j\frac{2\pi}{E}\nu t} + 
%      \hat{f}_{\theta}^{*}[\nu]e^{-j\frac{2\pi}{E}\nu t} \right) \notag \\
%   &   = \frac{2}{E} \left(\re(\hat{f}_{\alpha}[\nu])\cos \left({2\pi \nu t}/{E}\right) - \notag\\
%   &  \qquad \qquad \qquad \qquad \qquad \qquad \im (\hat{f}_{\alpha}[\nu]) \sin \left({2\pi \nu t}/{E} \right) \right) \notag \\
%   & = \frac { 2 } { E } | F [ \nu ] | \cos \left( \left(  { 2 \pi } / { E \epsilon } \nu \right)  t  \epsilon + \arg ( F [ \nu ] ) \right)
% \end{align}
\begin{align}
\label{fvt1}
     \mathcal{C}^{\nu}_{\alpha} & = \frac{1}{E}\left(\hat{f}_{\alpha}[\nu] e^{j\frac{2\pi}{E}\nu t} + 
     \hat{f}_{\alpha}^{*}[\nu]e^{-j\frac{2\pi}{E}\nu t} \right) \notag \\
  &   = \frac{2}{E} \left[\re(\hat{f}_{\alpha}[\nu])\cos \left({2\pi \nu t}/{E}\right) - \im (\hat{f}_{\alpha}[\nu]) \sin \left({2\pi \nu t}/{E} \right) \right] \notag \\
  & = \frac { 2 } { E } | F [ \nu ] | \cos \left[ \left(  { 2 \pi\nu} / { E \epsilon } \right)  t  \epsilon + \arg ( F [ \nu ] ) \right],
\end{align}
where $\re(\cdot)$ and $\im(\cdot)$ denote taking the real part and the imaginary part of a imaginary number, respectively. Hence, by integrating the signals of all frequencies $\nu = 0:E-1$
\begin{equation}\label{sumfreq}
    f_{\alpha}[t] = \sum_{\nu = 0}^{E-1}\mathcal{C}^{\nu}_{\alpha}[t]e^{j\frac{2\pi}{E}t\nu} .
\end{equation}
Equation~\eqref{pro2} can be obtained by substituting Eq.~\eqref{fvt1} into Eq.~\eqref{sumfreq}.
\end{proof}

\subsubsection{Initial and training state of the node}\label{secstartup}
We build the GCN process on the mathematical model given by Sec.~\ref{secmath}. In this section, we first obtain the initial state of the graph signal by blocking the signal transmission, which is stated in the following proposition.

\begin{proposition}\label{proposition2}
The initial signal of node $\alpha$ is
\begin{equation}\label{eqpro3}
f_{\alpha}[0]=  2\bar{\theta}_0 \mathrm{Sum}\left[ u(\alpha) \right],
\end{equation}
\end{proposition}

\begin{proof}
According to the general GCN forward propagation~\cite{52}, the embedding feature of $\mathcal{G}$ is
\begin{equation}\label{eqgeneralGCN}
    f^e[\cdot]=\sigma \left( U g _ { \theta } ( \Lambda ) U ^ { \top } f \right),
\end{equation}
where $g _ { \theta } ( \Lambda ) = \mathrm{Diag}\left(\theta_1,\ldots,\theta_N \right)$ is the trainable diagonal matrix, hence, the output of GCN for node $\alpha$
is 
\begin{equation}\label{eqsimplify}
    f^e_{\alpha}[\cdot]=u(\alpha) g _ { \theta } ( \Lambda ) U ^ { \top } {f}.
\end{equation}
% considering that there is no direct relationship between the node features distribution and the graph structure, so, 

Now we simplify Eq~\eqref{eqsimplify}. We denote $ g _ { \theta } ( \Lambda ) U ^ { \top }$ as a constant matrix (note that so far, message passing is not activated yet, i.e., the trainable parameters can be regarded as constants) $\Phi \in \mathbb{R}^{N\times d}$ where $d$ is the dimension of the feature, $\phi_i$ is the $i^{\text{th}}$ row of $\Phi$, $\phi_i(k)$ is the $k^{\text{th}}$ element of $\phi_i$. 
In $d=1$ case, 
\begin{equation}\label{eqd_1}
     f^e_{\alpha}[\cdot]= \sum_{i=1}^{N} \phi_{1}(i)u_i(\alpha),\ \stt,\ d=1,
\end{equation}
in $d>1$ case,
\begin{equation}
    f^e_{\alpha}[\cdot] = \left[\sum_{i=1}^{N} \phi_{1}(i)u_i(\alpha),\ldots,\sum_{i=1}^{N} \phi_{d}(i)u_i(\alpha)\right],\ \stt,\ d>1.
\end{equation}
Since our aim is to solve the qualitative problem of node localization, we only give the case of $d=1$, as the discovery making from observation: Eq.~\eqref{eqd_1} is homogeneous with Eq.~\eqref{ft}, moreover $\phi_i$ and $c_i$ both denote the changeable parameters, we conclude that Proposition~\ref{proposition1} and Proposition~\ref{proposition2} are applied to the training phase of GCN. 

We consider the initial state as the non-transmission state, that is, for node $\alpha$, $f_{\alpha}[0]=\ldots=f_{\alpha}[E-1]$, hence, the transmission of graph signal can be blocked by letting $c_i=e^{-\lambda_i t}$. By regarding the initial status of GCN as $E=1$, $t=0$, Eq.~\eqref{eqpro3} can be obtained by substituting Eq.~\eqref{pro1}
into Eq.~\eqref{pro2}, \stt, $E=1$ and $t=0$.
\end{proof}

Next, we active the message passing on $\mathcal{G}$. As $\bar{\theta}$ denotes constants, Eq.~\eqref{pro1} can be concluded by
\begin{equation}\label{eqpro1a}
    \hat{f}_{\alpha}[\nu] = \sum_{t=0}^{E-1}\left(\sum_{i=1}^{N} \bar{c}_i e^{\lambda_{i}t} u_i(\alpha) \right) e^{-j\frac{2\pi}{E}t\nu},
\end{equation}
where $\bar{c}_i=\bar{\theta}_t c_i$. By substituting Eq.~\eqref{eqpro1a} into
Eq.~\eqref{pro2}, the final expression of node signal is:
\begin{multline}\label{eqfinal}
       f_{\alpha}[t] =\lim_{\epsilon \to 0}\Bigg\{ \sum_{\nu=0}^{E-1}\frac{2}{E} \Bigg| 
    \sum_{t=0}^{E-1}\Big(\sum_{i=1}^{N} \bar{c}_i e^{\lambda_{i}t} \underbrace{u_i(\alpha)}_{\text{from } \mathcal{E}} \Big) e^{-j\frac{2\pi}{E}t\nu}
    \Bigg|\\
    \cos \left[ \left(  { 2 \pi\nu} / { E \epsilon } \right)  t  \epsilon + \arg ( \hat{f}_{\alpha}[\nu] ) \right] e^{j\frac{2\pi}{E}t\nu} \Bigg\}.
\end{multline}

Equation~\eqref{eqfinal} gives the signal model of all nodes in the unified reference frame, by observing, for node $\alpha$, among all the initial conditions related to $\mathcal{G}$, only $u(\alpha)$ appears in the process of node signal change. In other words, in the elastic system brought by end-to-end training, only $u(\alpha)$ is a controllable factor, so we simplify the above equation to
\begin{equation}
    f_{\alpha}[t] = \mathcal{F}\left[ u(\alpha) \right],
\end{equation}
and the feature vibrate of node $\alpha$ under the end-to-end training is
\begin{align}\notag
     &\texttt{Initial: }\ \  2\bar{\theta}_0 \mathrm{Sum}\left[ u(\alpha) \right]\\
    &\texttt{Training: } \mathcal{F}\left[ u(\alpha) \right] \notag.
\end{align}
Thus, the feature vibrate of node $\alpha$ driven by GCNs can be quantified, and it can always contain a fixed factor $u(\alpha)$ while message passing, as GCNs quantify the message passing in training phase is part according to the input $\Delta$, GCNs deconstruct $\Delta$ to Laplacian matrix $U$ and further quantify the feature vibrate of node $\alpha$ by $u(\alpha)$, i.e., GCNs locate node $\alpha$ according to $u(\alpha)$.

\subsection{Basic Model for AN-GCN}
According to Theorem~\ref{claim}, we denote each row of the Laplacian matrix as an independent generating target. We use the spectral graph convolution (without any constraints and simplification rules)~\cite{22} as the basic model, and take it as an encoder. We use an extra fully connected neural network decoder. That is, denote the basic forward propagation as:
\begin{subequations}\label{eqbeforegen}
\begin{equation}\label{eqbeforegenA}
f^e = \sigma \left( U g _ { \theta } ( \Lambda ) U ^ { \top } f \right)
\end{equation}
\begin{equation}\label{eqbeforegenB}
y_{out} = f^e W^D.
\end{equation}
\end{subequations}

In Eq.~\eqref{eqbeforegen}, $U$ contains the information of the edges in the graph, we will replace it with a matrix generated from Gaussian noise. Eq.~\eqref{eqbeforegen} is the basic signal forward propagation. Next, we will describe how to improve Eq.~\eqref{eqbeforegen}, thus we can let the base model accommodate the generated node position.

\subsection{Generating the Node Position From Noise}\label{secgen}

If $u(n)$ is completely generated by noise, the specific points will keep anonymous. Thus we take $u(n)$ as the generation target. The output of the generator is denoted as $u^G(n)$, which tries to approximate the underlying true node position distribution $u(n)$.

Next, we define a probability density function (PDF) for input noise. In order to enable the generator to locate a specific point, the input noise of the generator will be constrained by the position of the target point. So we define the input noise as having a PDF equal to that of Staggered Gaussian Distribution, the purpose is to make the noise not only satisfy the Gaussian distribution but also do not coincide with each other, thus letting the generated noises orderly distributed on the number axis.

\begin{proposition}[Staggered Gaussian Distribution]
\label{pro4}
. Giving a minimum probability $\varepsilon$, $N$ Gaussian distributions centered on $x=0$ satisfy $
\mathrm { P } ( x , n ) \sim \operatorname { Norm } ( 2 \sigma ( 2 n - N - 1 ) \sqrt { \log ( \sqrt { 2 \pi } \sigma \varepsilon ) } , \sigma^2 )
$, so that the probability density function of each distribution is greater than $\varepsilon$, where $\mathrm{ Norm }$ is the Gaussian distribution, $n$ is the node number, $\sigma$ is the standard deviation, and $\varepsilon$ is the set minimum probability.
\end{proposition}

\begin{proof}
Given a probability density function $h(x_p)$ of the Gaussian distribution $\mathrm{Norm}(\mu_p,\sigma^2)$, when $h(x_p)=\varepsilon$, 
\begin{equation}
    x_p=\mu_p \pm 2\sigma \sqrt{\log{\left(\sqrt{(2\pi)}\sigma\varepsilon\right)}}.
\end{equation}
Let $ 2\sigma \sqrt{\log{(\sqrt{(2\pi)}\sigma\varepsilon)}}=r$ as the distance from the average value $\mu_p$ to maximum and minimum value of $x_p$. Specify that each $x_p$ represents the noise distribution of each node. In order to make all the distributions staggered and densely arranged, stipulate $\mathrm{max}(x_p)=\mathrm{min}\left(x_{p+1}\right)$, and keep all distributions symmetrical about $x = 0$. So , when the total number of nodes is $N$, $\mu_1=(1-N)r$, $\mu_2=(3-N)r$, $\ldots$, $\mu_N=(N-1)r$, that is, $\mu_n=(2n-N-1)r=2 \sigma ( 2 n - N - 1 ) \sqrt { \log ( \sqrt { 2 \pi } \sigma \varepsilon ) }$.
\end{proof}

% \begin{Proof}\notag
% Given a probability density function $h(x_p)$ of the Gaussian distribution $\mathrm{Norm}(\mu_p,\sigma^2)$, when $h(x_p)=\varepsilon$, 
% \begin{equation}
%     x_p=\mu_p \pm 2\sigma \sqrt{\log{\left(\sqrt{(2\pi)}\sigma\varepsilon\right)}}
% \end{equation}

% Let $ 2\sigma \sqrt{\log{(\sqrt{(2\pi)}\sigma\varepsilon)}}=r$ as the distance from the average value $\mu_p$ to maximum and minimum value of $x_p$. Specify that each $x_p$ represents the noise distribution of each node. In order to make all the distribution Staggered and densely arranged, stipulate $\mathrm{max}(x_p)=\mathrm{min}\left(x_{p+1}\right)$, and keep all distributions symmetrical about $x = 0$. So , when the total number of nodes is $N$, $\mu_1=(1-N)r$, $\mu_2=(3-N)r$, $\ldots$, $\mu_N=(N-1)r$, that is, $\mu_n=2 \sigma ( 2 n - N - 1 ) \sqrt { \log ( \sqrt { 2 \pi } \sigma \varepsilon ) }$

% \end{Proof}

The process of generating sample $u^G(v)$ from Staggered Gaussian noise $Z_v\sim P(x,v)$ is denoted as $u^G(v)=G\left(Z_v \right)$, and $U$ generated by $G$ is denoted as $U^{G}$. The process of generator is shown in Fig.~\ref{fig4}.
\begin{figure}
\centering
\includegraphics[width=7cm]{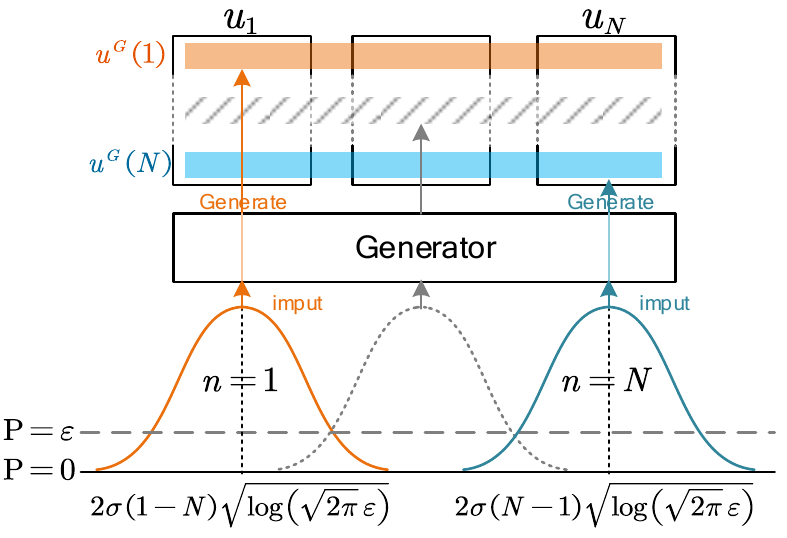}
\caption{Generator with Staggered Gaussian distribution as input.} \label{fig4}
% \vspace{-0.5cm} 
\end{figure}

\subsection{Detecting the Generated Position}\label{secdis}
\begin{figure*}[htb]
\centering
\includegraphics[width=0.70\textwidth]{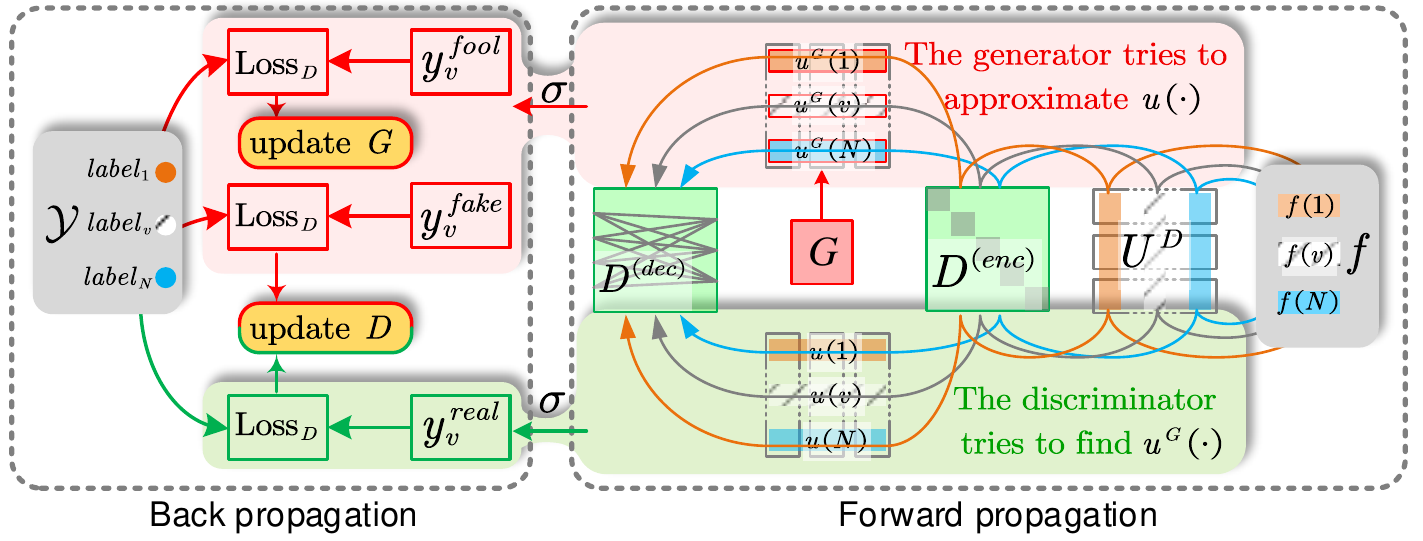}
\caption{The schematic illustration of AN-GCN.}
\label{fig5}
% \vspace{-0.5cm}
\end{figure*}
\begin{algorithm}[htb]%算法1
\caption{AN-GCN}
\label{ag1}
{\begin{algorithmic}[1]
\Require Weights of $D$: $\theta^{D}$,include $ g_\theta^{D,enc}(\Lambda)$ is used for encoding, $\theta^{D,dec}$ is used for decoding. Weights of $G$: $\theta^G$. Training epoch of $D$ and $G$.
\State Initialize $U^D=U$
\For{epoch in Train epochs for $D$}
    \State Random sample a node $v$
    \State Obtain noise $Z_v$ from ${ P } ( z , v )$
    \State $u^G(v) \gets G(Z_v^D)$ // \textit{Generate} 
    \State $U^D_{\cdot v} = U^D_{\cdot v} + q\left(u^G(v)-U^D_{\cdot v}\right)$ // \textit{Linear approximation}
    \State Calculate fake label of $v$ by Eq.~\eqref{eqfakerealA}
    %// \textit{Ensure correct position}
    \State Calculate real label of $v$ by Eq.~\eqref{eqfakerealB}
    %// \textit{Ensure correct prediction}
    \State Calculate $\nabla_{update}^D$ by Eq.~\eqref{eqgrad}
    \State $\theta^{D} \gets \theta^{D}-\nabla_{update}^D$
    \For{epoch in Train epochs for $G$}
        \State Obtain noise  $Z_v$ from ${ P } ( z , v )$
        \State Calculate $y^{fool}_v$ by Eq.~\eqref{eqfool}
        \State Calculate $\nabla_{update}^G$ by Eq.~\eqref{eqgradG}
        \State $\theta^{G} \gets \theta^{G}-\nabla_{update}^G$
    \EndFor
\EndFor
\Ensure Trained generator weights $\theta^G$.
\end{algorithmic}}
% \vspace{-0.3cm}
\end{algorithm}

After proposing the generation of $u^G(n)$, we need to set discriminator $D$ to evaluate the quality of $u^G(n)$ generated by generator $G$, and design the two-player game in order to make $u^G(n)$ can complete the accurate node classification.

To ensure the accuracy of classification, we use classification quality as an evaluation indicator to drive the entire adversarial training, so in the game process, $D$ can not only distinguish the adversarial samples generated by $G$ (non-malicious, used for anonymize node), but also ensure the accuracy of node classification through the adversarial training with the classification accuracy as the final optimization objective. Specifically, $D$ is divided into two parts: a diagonal matrix with trainable parameters $D_{\text{enc}}$ used for encode (embed) the graph, and a parameter matrix $D_{\text{dec}}$ used for decode, their are $ g_\theta^{D,dec}(\Lambda)$ and $\theta^{D,enc}$, respectively. Thus the forward propagation becomes
\begin{equation}\label{equut}
    f^e = \sigma \left( U D_{\text{enc}} U ^ { \top } f \right) D_{\text{dec}}.
\end{equation}

Consider there are $U$ and $U^{\top}$ in the encoder (Eq.~\eqref{eqbeforegenA}), we mark them as different matrices:
\begin{equation}\label{eqmark}
    f^e = \sigma \left( U^{G} D_{\text{enc}} U ^ {D} f \right)D_{dec},
\end{equation}
that is, the corresponding $u(\cdot)$ to $U^G$ is generated by independent generation of different rows, and the $U^D$ on the right is approximate linearly according to the numerical change of left $U$. Consider that $U^{G}$ and $U^{D}$ changed dynamically while training, we use $U^{G,e}$ and $U^{D,e}$ to denote $U^{G}$ and $U^{D}$ in epoch $e$ respectively. While generate row $l$ of $U^G$, the corresponding column of $U^D$ (denote as $u^{D}_l$) will linearly approximate towards $u^{G}(l)$. Specifically, the general term formula for its value in training epoch $e$ is 
$$u^{D,e}_l=\left\{ \begin{array} { l } u_l,\ \stt \ e=1 \\ u^{D,e-1}_l+q\left(u^{G,e}(l)-u^{D,e-1}_l\right),\ \stt \ e>1 \end{array} \right.,$$
where $u^{G,e}(l)$ is the generated $u^{G,e}(l)$ in epoch $e$, $q$ is the custom linear approximation coefficient. Thus, $U^{D,e}$ is gradually approaching the generator matrix $U^{G,e}$ during the training phase. With the increase of training epochs, the collaboration among two $U$ in Eq~\eqref{equut} (i.e., $U$ and $U^{\top}$) and $G$ are shown in the Fig.~\ref{figUUT}.

\begin{figure}[htb]
\centering
\includegraphics[width=0.4\textwidth]{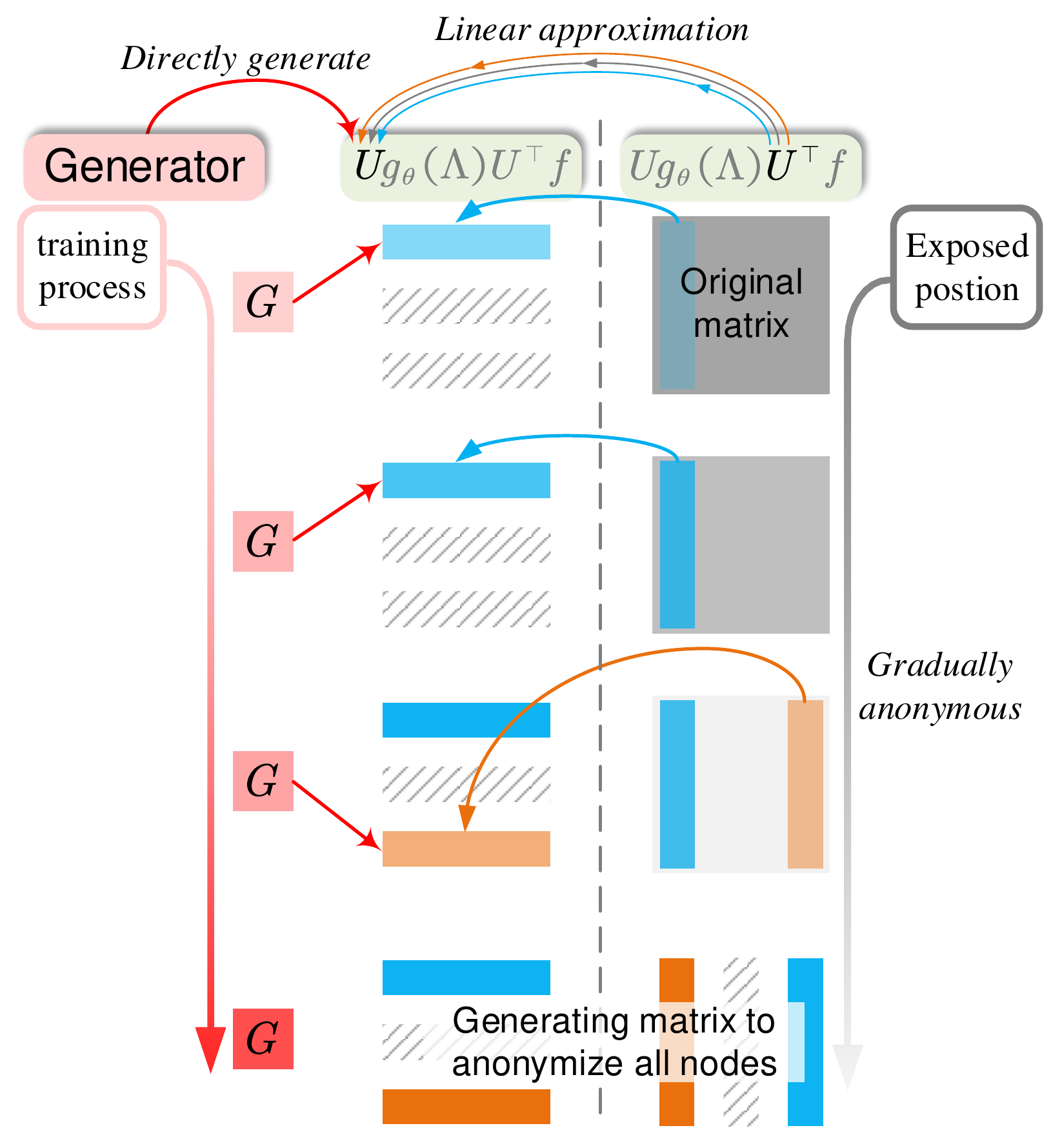}
\caption{In the training phase, $U^D$ gradually approximates to $U$ linearly, so as to realize anonymity.}\label{figUUT}
% \vspace{-0.5cm}
\end{figure}

After all, the forward propagation of the basic model is improved as Fig.~\ref{figeq}.

\begin{figure}[htb]
\centering
\includegraphics[width=0.4\textwidth]{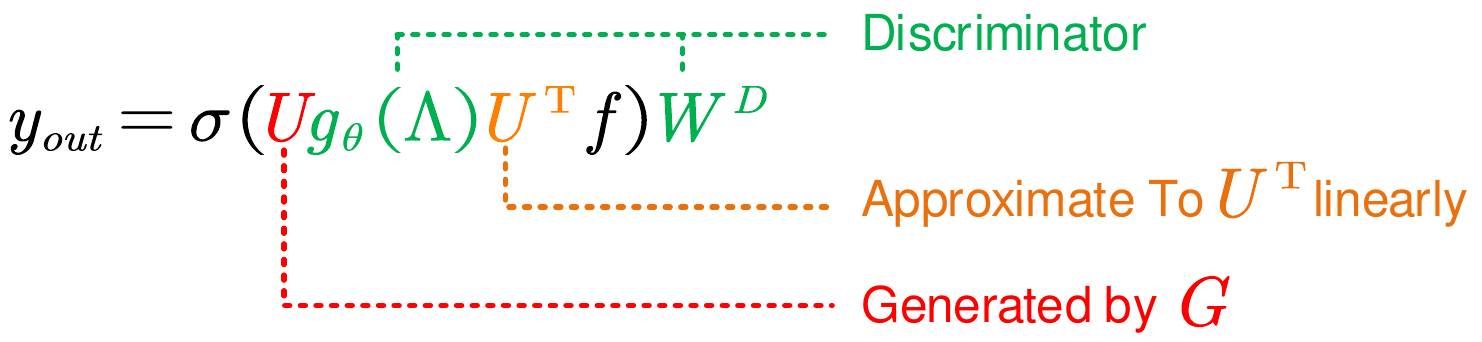}
\caption{Forward propagation of AN-GCN.}\label{figeq}
% \vspace{-0.5cm}
\end{figure}

So far, we have eliminated the edge information stored in AN-GCN. Next, we will elaborate on how to make the nodes participate in the training phase anonymously, and the optimization of generator and discriminator, so as to ensure that both sides of the adversarial training have the clear objectives.

\subsubsection{Anonymous training}

The process of node anonymous forward propagation is as follows. Let $f ^ {e,G} ( v ) $ and $f ^ {e,D} ( v )$ denote the embedding of the node $v$ when using $u^G(v)$ and $u(v)$ locate nodes at epoch $e$, respectively. In each training epoch $e$, first we obtain $u^G(v)$ and $U^{D,e}$ corresponding to the target node $v$, that is, generate the $u^G(v)$ from the Staggered Gaussian noise (Proposition.\ref{pro4}) corresponding to node $v$, and update the corresponding column of $U^{D,e-1}$ to $U^{D,e}$. Second, we use $D_{\text{enc}}$ and $D_{\text{dec}}$ to evaluate the quality of this epoch of generators, by getting the node embedding $f^{e,G}(v)$ of $v$ through $D_{\text{enc}}$, decode $f^{e,G}(v)$ through $D_{\text{dec}}$ to get the label possibility. 

Consider we need to calculate specific node embedding, based on Eq.~\eqref{eqbeforegenA}, the embedding of node $v$ is
\begin{equation}\label{eqlocate}
    f^e(v) = \sigma \left( u(v) g _ { \theta } ( \Lambda ) U ^ { \top } f \right).
\end{equation}

\subsubsection{Discriminator optimization}

Discriminator optimization aiming at reducing the classification accuracy of the generator, while improving the classification accuracy of the discriminator. The classification accuracy is quantified by the classification probability (i.e., soft label). For the node $v$ with generated position, the soft label of node $v$ is denoted as $y^{fake}_v$, which can be calculated according to the generated node position $u^G(v)$, the linear approximation matrix $U^D$ and the discriminator of current epoch. After that, in order to make the discriminator detect the node $v$ (which with generate position), the real label $y^{real}_v$ is calculated by the general GCN~\eqref{eqgeneralGCN}. Thus, $y^{fake}_v$ and $y^{real}_v$ are given as:
\begin{subequations}
\begin{equation}
     y^{fake}_v= u^G(v) D_{\text{enc}} U^D f(v)D_{\text{dec}} \label{eqfakerealA},
\end{equation}
\begin{equation}
    y^{real}_v= u(v) D_{\text{enc}} U f(v) D_{\text{dec}}\label{eqfakerealB}.
\end{equation}
\end{subequations}

% \subsubsection{Loss function formulation}

In the training phase of discriminator, in order to detect nodes with generated position, the discriminator not only enhances its ability of correct classification, but also classifies nodes with generated position into other random categories, so as to reduce the performance of the generator. In the discriminator training phase, the discriminator will try to classify the normal nodes as correct labels, and the generated nodes as wrong labels. Specifically, in the discriminator training phase, the discriminator will try to classify the normal nodes as correct labels and the nodes with generated positions as wrong labels. The performances of the generator and discriminator are quantified by the loss function, thus, we need to calculate the loss functions of $y^{real}_v$ and $y^{fake}_v$ respectively. The loss function of $D$ is designed as:
\begin{equation*}
\mathrm{Loss}_D(y)=\left\{\begin{array} { l } \mathbb{C}[\mathbb{S}(y),label_v],\stt\  y=y^{real}_v\\ \mathbb{C}[\mathbb{S}(y),\mathrm{SP}(\left\{label_\gamma|\gamma\neq v \right\})],\stt\  y= y^{fake}_v \end{array} \right.,
\end{equation*}
where $label_i$ represents the true label of node $i$ (one-hot), $\mathbb{C}$ denotes Cross entropy function, $\mathbb{S}$ denotes Sigmoid function, $\mathrm{SP}$ denotes Random sampling. 

Finally, according to the loss $\mathrm{Loss}_D(y)$, we update the weight of $D$ according to the gradient:
\begin{equation}\label{eqgrad}
    \nabla_{update}^D = \nabla_{g_\theta^{D,dec}(\Lambda),\theta^{D,enc}}\big[\mathrm{Loss}_D(y^{real}_v)+ \mathrm{Loss}_D(y^{fake}_v)\big].
\end{equation}

% he generator will try to generate a fake position, and hope that the discriminator of the current epoch can classify node $v$ as the correct label. I
\subsubsection{Generator optimization}

Next, aiming that letting $u^{G}(v)$ generated by $G$ can provide accurate classification, we train $G$. After re-sample noise $Z_v$, the label used to fool $D$ is calculated by
\begin{equation}\label{eqfool}
    y^{fool}_v = G(Z_v) D_{\text{enc}} U^D f D_{\text{dec}},
\end{equation}
in generation epoches, the generator attempts to classify the node $v$ as the correct label, thus the loss function of $D$ is designed as:
\begin{equation}
    \mathrm{Loss}_D(y)=\mathbb{C}\big[\mathbb{S}(y),label_v\big],\stt\ y=y^{fool}_v.
\end{equation}

Finally, in order to make the output of $G$ be correctly classified by $D$, we update the weight of $G$ according to the output of $D$, note that $D$ is frozen in the training phase of the generator.
\begin{equation}\label{eqgradG}
    \nabla_{update}^G = \nabla_{\theta^G}\big[\mathrm{Loss}_D(y^{fool}_v)\big]
\end{equation}

% Formally, take $u$ as input, we use $D(u;\theta^D)$ to express the output under the weight $\theta^D=\{g_\theta^{D,dec}(\Lambda),\theta^{D,enc}\}$, $G$ and $D$ are playing the following two-player minimax game with value function $V (G, D)$:
% \begin{multline}
% \max \limits_{\theta^G} \min \limits_{\theta^D}\sum_{v=1}^N\bigg\{ \mathbb{E}_{u \in  U}\mathrm{Loss}_D\big[D(u;\theta^D)\big]\\
% +\mathbb{E}_{u \sim G(Z_v;\theta^G)}\big[1-\mathrm{Loss}_D(D(u;\theta^D))\big]\bigg\}.
% \end{multline}

The schematic illustration of AN-GCN  is represented by Fig.~\ref{fig5},  and the algorithm process of AN-GCN is given in the Algorithm ~\ref{ag1}.
\begin{figure*}[htb]
\centering
\includegraphics[width=\textwidth]{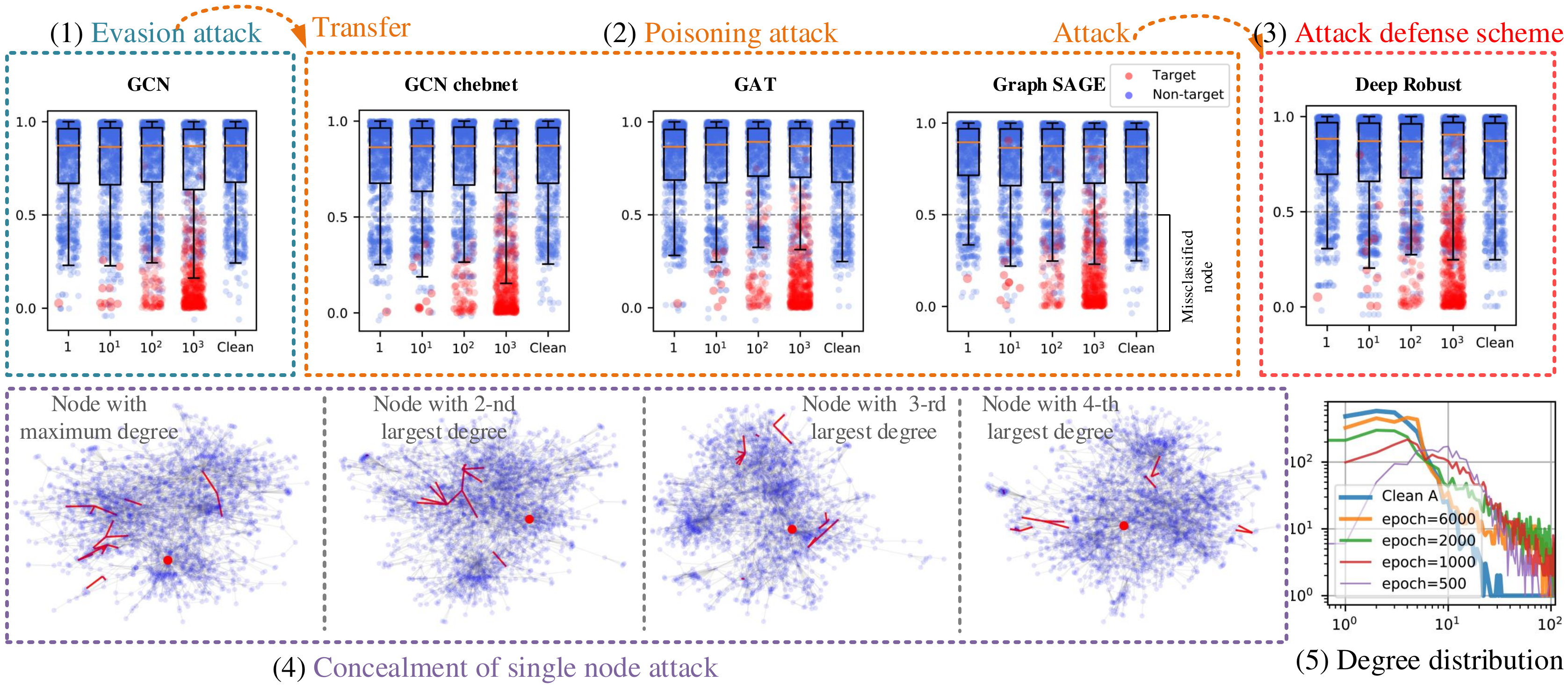}
% \subfigure{\includegraphics[width=3cm]{table1.pdf}}
\caption{The evaluation results of G-EPA (Eq.~\eqref{eqattack1}). 1. Evaluation $\mathit{1}$) The horizontal axis is the number of attack nodes, and the vertical axis is the scaling index according to the loss function. The more similar the distribution of non-target nodes and clean graph, the less misclassified non-target nodes will be. It can be seen that with the increase of the number of attacks, almost all the target nodes are misclassified, while the non-target nodes almost keep the original classification results. 2. Evaluation $\mathit{2}$) After $\mathcal{G}_{victim}$ was transferred to different models and retrained, the performance of each model was evaluated. It can be seen that all models are successfully attacked, which shows that the attack method represented by Eq.~\eqref{eqattack1} can transfer to other models. 3. Evaluation $\mathit{3}$) it can be seen that a large number of target nodes are misclassified, thus prove the defense ability of the model decreases. 4. Evaluation $\mathit{4}$) The red dot is the target node, and the red line is the modified connection. As can be seen that we have not directly modified the node, and the total number of perturbations is insignificant, so the attack has well performance on concealment. 5. Evaluation $\mathit{5}$) With the increase of training epoch, the degree distribution is more and more like the degree distribution of clean $A$, that is to say, $\mathcal{G}_{victim}$ will eventually tend to be stable.}
\label{figaa}
\end{figure*}
\subsection{Security Analysis}
In the training phase, as $U^D$ gradually approximates to $U$ linearly, the forward propagation of AN-GCN can be
\begin{sequation}\label{eqFPANGCN}
 \mathcal{Y} = \left[ \begin{array}{ l } G(Z_1)\\ \cdots \\ G(Z_N) \end{array}\right]  g _ { \theta } ( \Lambda ) \left[G(Z_1)^{\top},\cdots,G(Z_N)^{\top}\right]f D_{dec},
\end{sequation}hence we simplify Eq.~\eqref{eqFPANGCN} to
\begin{sequation}\label{eqsimFP}
 \mathcal{Y} = \mathbf{AN\text{-}GCN}\left(\mathbb{N};\theta^G,\theta^D\right),
\end{sequation}where $\mathbb{N}=\{1,\ldots,N\}$ is the node indices set. Comparing to the forward propagation of the existing GCNs (Eqs.~\eqref{eq1} and~\eqref{eqaag}), the position of nodes in AN-GCN are invisible, i.e., AN-GCN can anonymize nodes, thus disables edge perturbations for attackers.

\section{Evaluation}\label{seceva}
We report results of experiments with two widely used benchmark data sets for node classification, which include CORA~\cite{62,23}, CITESEER~\cite{64}, either used with extensive research of graph adversarial attack and defense~\cite{59,20,61}.

\subsection{Evaluation for the General Edge-Perturbing Attacks}

By observing the effectiveness of the attack model, we further prove that the exposure of node position is the main reason for graph vulnerability (Eq.\eqref{eq15}).

The evaluation for the attack effect of Eq.~\eqref{eqattack1} includes five items: $\mathit{1}$) The effectiveness of the attack (represented by box plot, Let Eq.~\eqref{eqattack1} act on the Semi-GCN~\cite{14} which trained on the clean graph to get $\mathcal{G}_{victim}$, evaluate the classification of $\mathcal{G}_{victim}$). $\mathit{2}$) Transferability of attack(transfer $\mathcal{G}_{victim}$ to the training set of GraphSAGE, GAT, GCN ChebNet~\cite{22}, thus get the $\mathcal{G}_{victim}$ and evaluate the classification performance of $\mathcal{G}_{victim}$ by various model). $\mathit{3}$) Attack effect on the existing mainstream defense scheme (Deep Robust~\cite{37}). The process is the same as that of evaluation $\mathit{1}$), and the target model becomes Deep Robust), $\mathit{4}$) Concealment of a single node attack. The evaluation index is expressed as to whether there is an obvious perturbation in the 5th order neighborhood of the target node. $\mathit{5}$) The stability of graph structure after the attack (represented by degree distribution~\cite{2}). The results are shown in Fig.~\ref{figaa}. In particular, the results of the evaluation item $\mathit{1}$),$\mathit{2}$) and $\mathit{3}$) above are presented in Table~\ref{tab1}.
\begin{table}[htb]
    \centering
    \caption{the results of the evaluation item $\mathit{1}$,$\mathit{2}$ and $\mathit{3}$, Success rate represents the probability that the target node is successfully misclassified}\label{tab1}
\includegraphics[width=5.0cm]{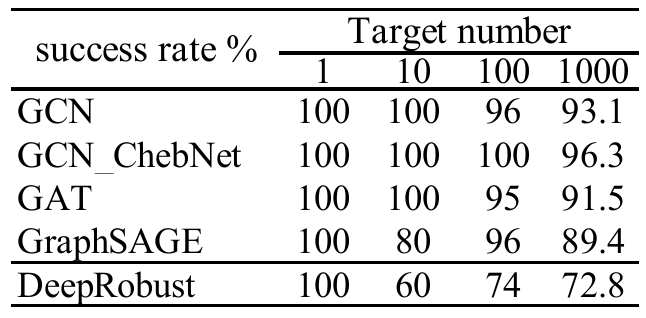}
\vspace{-0.5cm}
\end{table}

% \subsection{Evaluation for AN-GCN}

\subsection{Numerical experiment for the theoretical basis of AN-GCN}

Since Theorem~\ref{claim} is the theoretical basis of AN-GCN, in this section, we will verify the correctness of the Theorem~\ref{claim} by numerical experiment. We design two experiments to verify Theorem~\ref{claim}. In the first experiment, we observe the position deviation of nodes in the embedded space after the disturbance is applied to the corresponding $u(\cdot)$, if the deviation of node $v$ is significantly larger than that of other points while applied to $u(v)$, it shows that $u(v)$ is highly sensitive to the position of $v$. In the second experiment, we observe the change for the $u(\cdot)$ of neighbors after deleting a node. If with the increase of neighbor order (distance from the deleted node), the change of corresponding $u(\cdot)$ shows less significance, then it further shows that the specific location of $v$ is closely related to the value of $u(v)$. Both experiments prove the theoretical correctness of the Theorem~\ref{claim} collectively.

\begin{figure}
\centering
% \begin{minipage}{0.35\textwidth}
\includegraphics[width=0.5\textwidth]{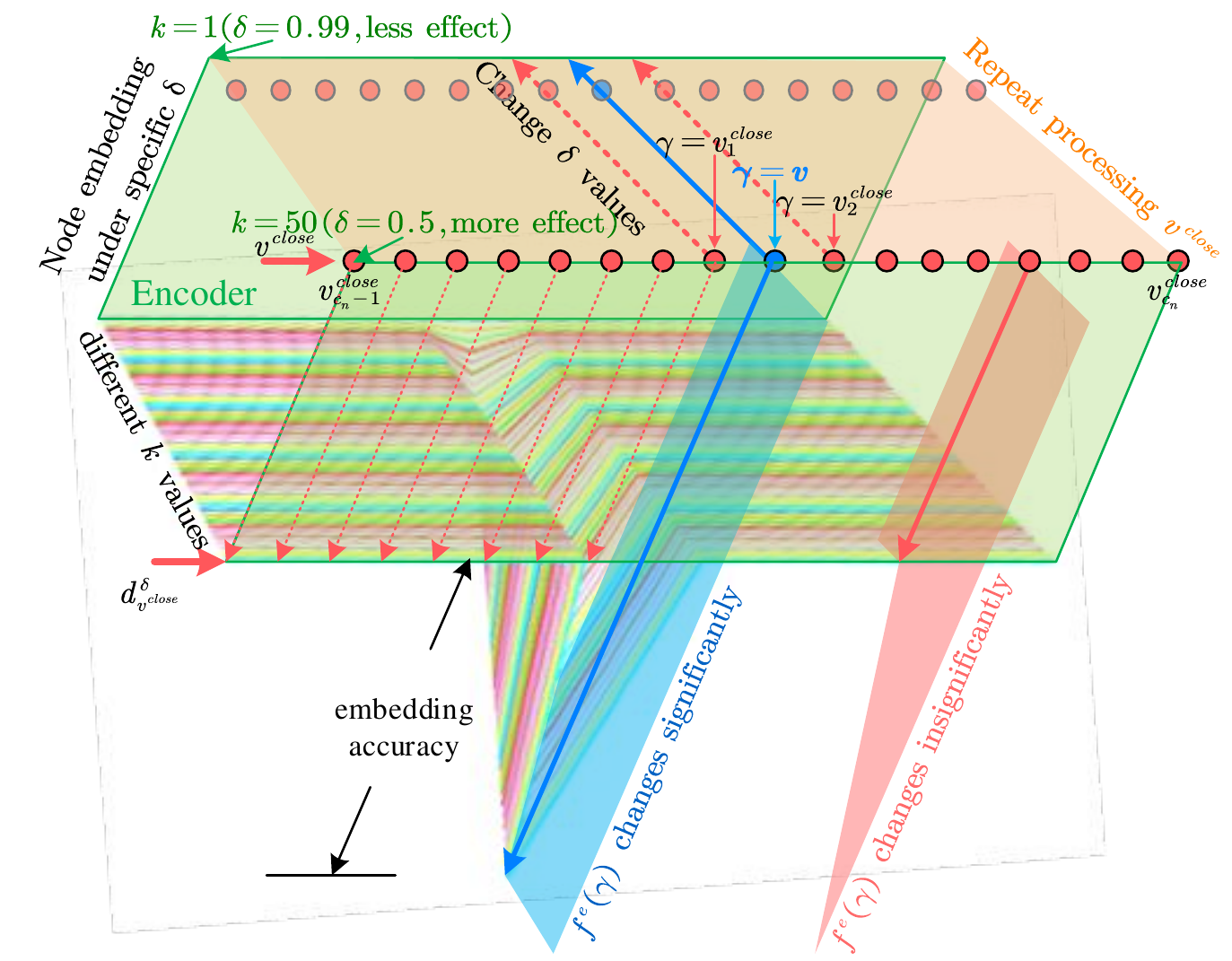}
% \end{minipage}%
% \begin{minipage}{0.15\textwidth}
\caption{The deviation of node embedding after disturb the $u(\cdot)$ corresponding to neighbors. Only when $\delta$ acts on the target point $v$, the embedding accuracy will suddenly drop (measured by the Euclidean distance between $f^e(v)$ and $f ^ { e,\delta } ( v )$), and the other conditions will remain stable.}\label{fig99}
\vspace{-0.5cm}
\end{figure}

\begin{figure}
\centering
\subfigure{
\centering
\begin{minipage}{0.20\textwidth}
\includegraphics[width=\textwidth]{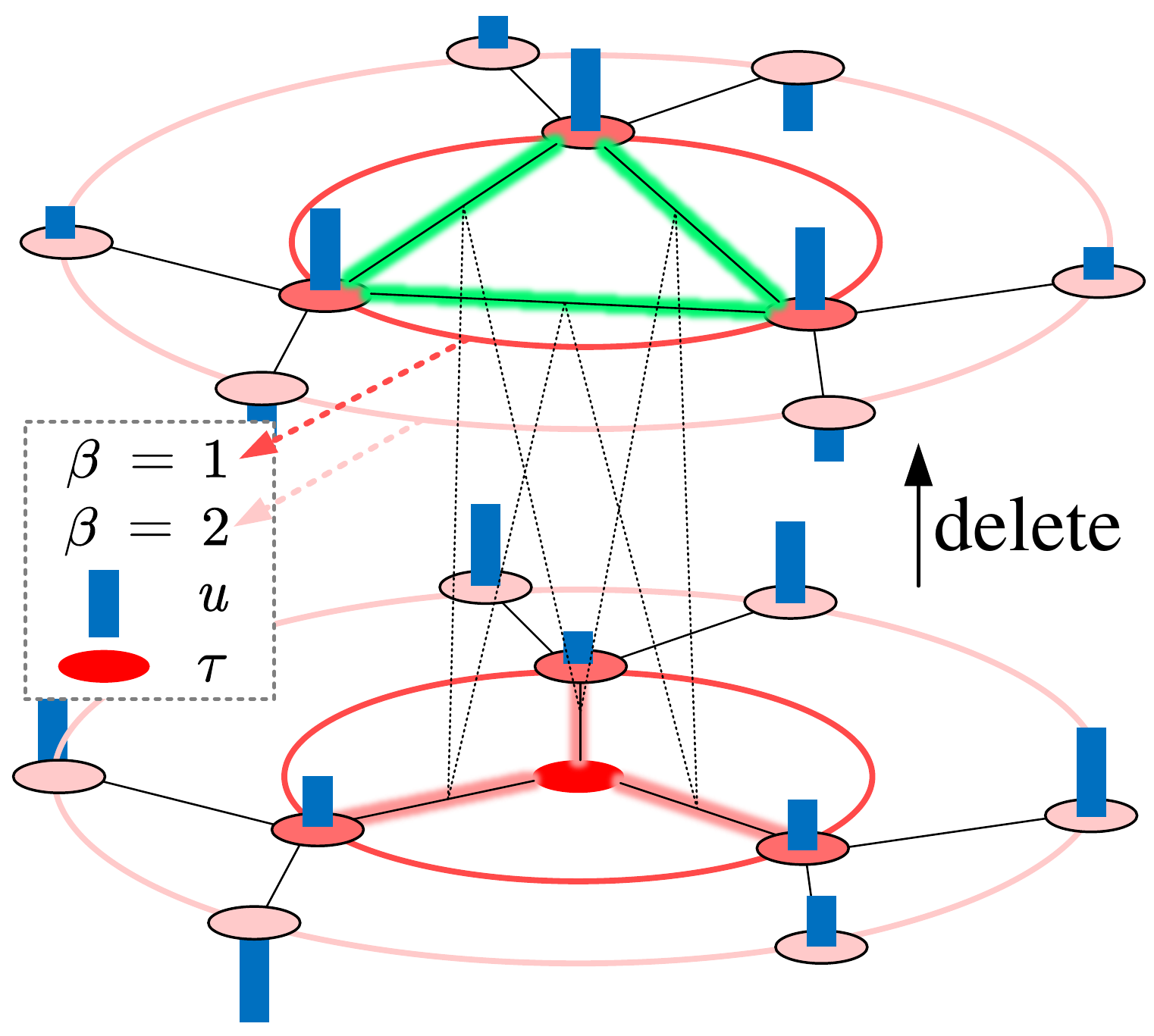}
\end{minipage}}
\subfigure
\centering
{\begin{minipage}{0.25\textwidth}
\includegraphics[width=\textwidth]{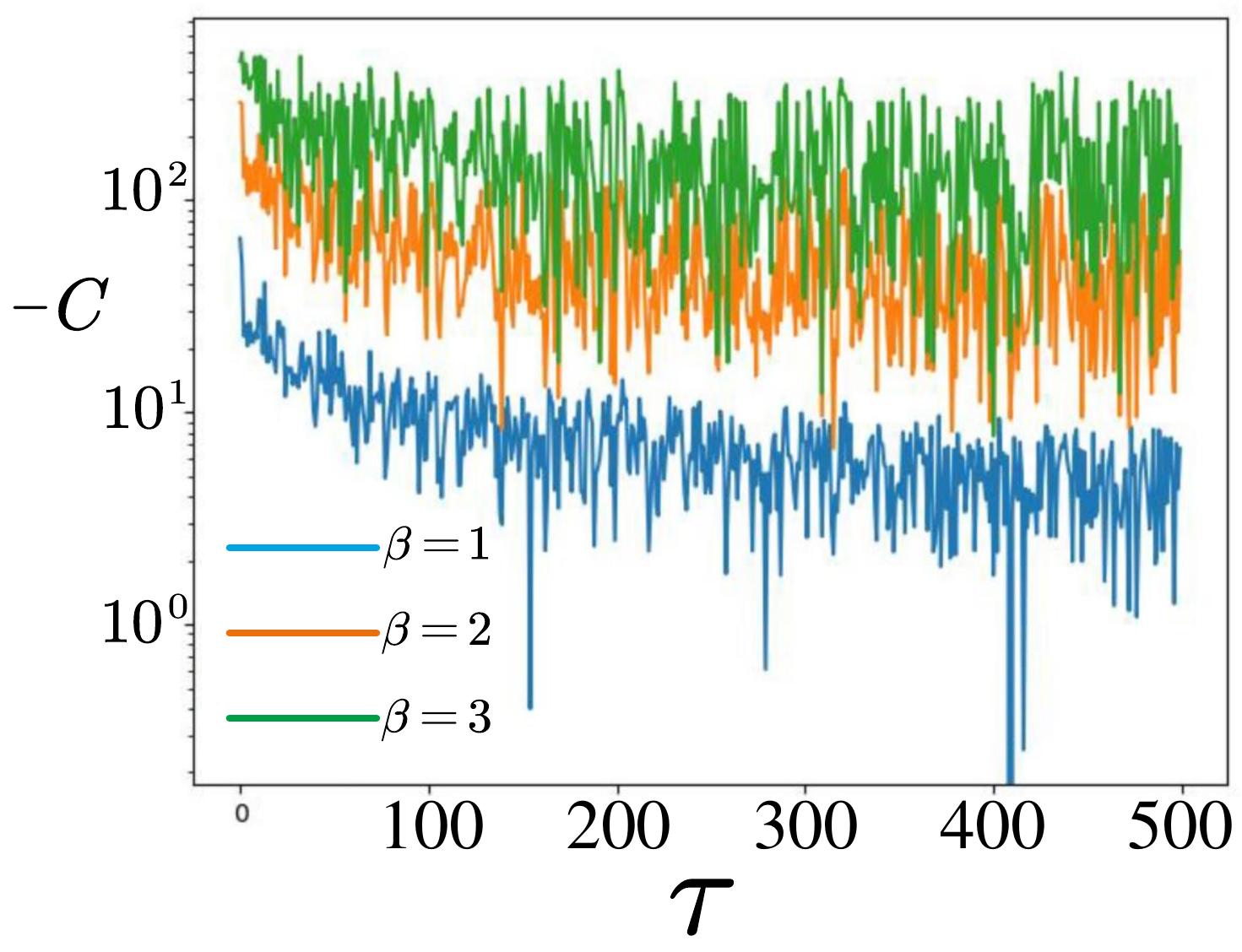}
\end{minipage}}
\caption{Deviation of $u(\cdot)$ in different positions after deleting nodes. After deleting $\tau$, the change of neighbor $u^{(d)}(\tau_{\beta}(\cdot))$ of different orders $\beta$ of $\tau$ (right). After deleting node $\tau$, the overall difference in the change of $u^{(d)}(\tau_{\beta}(\cdot))$ for each order neighbor is large, and the $u(\cdot)$ of the first-order neighbor $u^{(d)}(\tau_{\beta}(\cdot))$ has the largest change (the vertical axis is $-C$). In other words, after the node $\tau$ is deleted, the $u(\cdot)$ corresponding to its first-order neighbor the $u(\tau_{1}(\cdot))$ has changed significantly, while the $u(\tau_{2}(\cdot))$ and $u(\tau_{3}(\cdot))$ has changed less and showed a decreasing trend. Since deleting node $\tau$ significantly affects the position of its first-order neighbors, as the order increases, the degree of influence gradually decreases, so the change of its $u(\tau_{\beta}(\cdot))$  also gradually decreases.} \label{fig3}
\vspace{-0.5cm} 
\end{figure}

\begin{figure*}[ht]
\centering
\includegraphics[width=\textwidth]{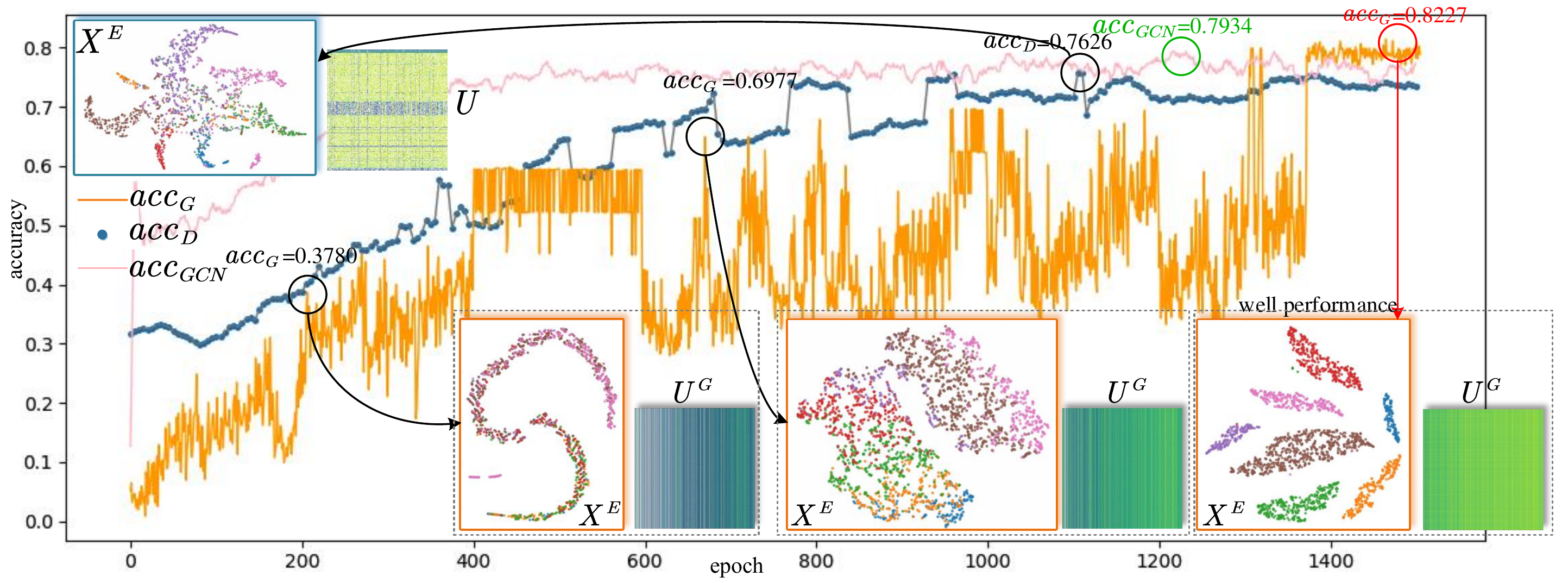}
\caption{Accuracy and visualization of node embedding during training, and visualization of node embedding. It can be seen that $acc_D$ and $acc_G$ are rising at the same time. When epoch $>$ 1400, $acc_G$  remains at a high and stable state. We select $G$ at epoch $= 1475$ as the final model selection, and the node embedding accuracy is 0.8227. In the 1500 epochs of training, the highest value of $acc_{GCN}$ is 0.7934. The experimental results show that under the same convolution kernel design, AN-GCN not only effectively maintains the anonymity of the node, but also has a higher detection accuracy than only GCN 0.0293 higher.}
\label{fig6}
% \vspace{-0.1cm} 
\end{figure*}

\subsubsection{The influence of $u(\cdot)$ on the node embedding}
By selecting a target node $v$, the deviation of $v$ in the embedding space is observed by applying a small disturbance to the $u(\cdot)$ for neighbors of $v$. Let $\mathrm{Nbor}_{1^{\text{th}}}(v) = \{v_1(1),\ldots,v_1(c_v)\}$ denote the $\text{1}^{\text{st}}$-order neighbors of $v$, given a factor $\delta$, the embedding deviation of $v$ is calculated by acting $\delta$ on $\mathrm{Nbor}_{1^{\text{th}}}(v)$ \textit{successively}, thus in every acting turn, given a target $v^{close}_i$ original $U$ becomes a disturbed matrix $\hat{U}^{\delta,c_v}$, where all rows in $\hat{U}^{\delta,c_v}$ satisfy 
\begin{equation}
     \hat{u}^{\delta,c_v}(i) = \left\{ \begin{array} { c }
    u(i), \ \stt \ i = v_1(i) \\
    \delta u(i),\  \stt \ i \neq v_1(i)
    \end{array}\right.,
\end{equation}
then embed $f$ to embedding space by $\hat{U}^{\delta,c_v}$ and a well-trained $g _ { \theta } ( \Lambda )$, by using Eq.~\eqref{eqbeforegenA}, embedding features under perturbing is
\begin{equation}
    f^{e,\delta} = \hat{U}^{\delta,c_v} g _ { \theta } ( \Lambda ) (\hat{U}^{\delta,c_v})^{\top}f.
\end{equation}
Further, we calculate the Euclidean distance~\cite{68} between $f^{e,\delta}$ and $f^e$ (Eq.~\eqref{eqbeforegenA})
\begin{equation}
    d_{v}^{\delta} =  \| f^{e,\delta} - f^e   \| _2 ,
\end{equation}
while changing $\delta$, using Chebyshev polynomial~\cite{22} as the convolution kernel, and Cora as the test data set, set $c_v=14$, $\delta=1-\frac{k}{100},k\in\left\{1,\ldots,50\right\}$, the result is shown as Fig.~\ref{fig99}, it can be seen $u(v)$ has a greater impact on the embedding of node $v$.

\subsubsection{The influence of the position changing of the node on $u(\cdot)$}Delete a node $\tau$ in the graph $\mathcal{G}$ to obtain the deleted graph $\mathcal{G}^{(d)}_{\tau}$, and calculate the laplacian matrix $L^{(d)}_{\tau}\sim\mathbb{R}^{(N-1)\times(N-1)}$ and matrix eigenvalues $U^{(d)}_{\tau}=\{u^{(d)}(1),\ldots,u^{(d)}(N-1)\}\sim\mathbb{R}^{(N-1)\times(N-1)}$. To keep $\mathcal{G}^{(d)}_{\tau}$ will connected, all edges connected to $\tau$ will be re-connected by traversal. Specifically, we stipulate that $\omega_{ij}$ is the weight of connecting nodes $i$ and $j$ in $\mathcal{G}$, and $\omega_{ij}^{(d)}$ corresponds to graph $\mathcal{G}^{(d)}_{\tau}$. The calculation method of all edge weights in $\mathcal{G}^{(d)}_{\tau}$ is:
\begin{equation}
\omega _ { i j } ^ { ( d ) } = \left\{ \begin{array} { c } \omega _ { i j } , \omega _ { \tau i } = 0 \text { or } \omega _ { \tau j } = 0 \\ \omega _ { i j + } \frac { \omega _ { \tau i } + \omega _ { \tau j } } { 2 } , \omega _ { \tau i } \neq 0 , \omega _ { \tau j } \neq 0 \end{array} \right..
\end{equation}

After obtaining the fully connected $\mathcal{G}^{(d)}_{\tau}$, we recalculate corresponding $U^{(d)}_{\tau}$, and query all $\beta^{\text{th}}$-order neighbors of $\tau$: $\mathrm{Nbor}_{\beta^{\text{th}}}(\tau)=\{\tau_{\beta}(1),\tau_{\beta}(2),\ldots\}$, and obtain a list $u(\tau_{\beta}(1)), \ldots, u(\tau_{\beta}(N-1))$ after deleting edges, which is denote as $u^{(d)}(\tau_{\beta}(\cdot))$. The quantitative representation of the change between the two is as follows:
\begin{small}
\begin{multline}\label{eq21}
    \mathrm{C}\left[u\left(\tau_{\beta}(\cdot)\right),u^{(d)}\left(\tau_{\beta}(\cdot)\right)\right]\\
    =\sum_{i=1}^{N-1}\left[\log{|u_i\left(\tau_{\beta}(\cdot)\right)|}^2-\log{|u_i^{(d)}\left(\tau_{\beta}(\cdot)\right)|}^2\right].
\end{multline}
\end{small}
Replace different $\beta$ and calculate $\mathrm{C}(\cdot)$, the result is as shown in Fig.~\ref{fig3}. We select the first 500 nodes according to the number of connections from large to small. Figure~\ref{fig3} proves that the position of node $\tau$ is inseparable from $u(\tau)$.

\subsection{Evaluation for the effectiveness of AN-GCN}

Next, we evaluated AN-GCN. We mainly evaluate the accuracy of AN-GCN. Its robustness to the perturbation of training set is as the secondary evaluation item, because the application scenario of AN-GCN is that the user can ensure that AN-GCN is trained on the clean graph, the poisoning attack scenario which training AN-GCN on the poisoned graph is not the main problem of this paper. But we still prove that AN-GCN shows the robustness to the perturbations in the training set. 

First, we evaluate the accuracy of AN-GCN on Cora dataset. We hope that AN-GCN can output accurate node embedding while keeping the node position completely generated by the well-trained generator, so we use the accuracy of node embedding to evaluate the effectiveness of AN-GCN. Since the generator does not directly generate the node embedding but completes the node prediction task by cooperating with the trained discriminator, we denote the accuracy of the generator $acc_G$ to be the accuracy of the node label through the position generated by $G$, and denote $acc_D$ as the accuracy of the discriminator. Further, we use  $acc_{GCN}$ to denote the accuracy of single-layer GCN (the convolution kernel adopts symmetric normalized Laplacian matrix~\cite{14}) with the same kernel of $D$ as a comparison. The results are shown in Fig.~\ref{fig6}. 

Second, we evaluate the robustness of AN-GCN to the attack on Cora and Citeseer dataset, by taking a comparative experimental model as Semi-GCN, GAT, and RGCN~\cite{38} (the state-of-the-art defense model). The results are presented in Table~\ref{tabaa}, $P\ rate$ is the perturbation rate. The meta-learning attack~\cite{36} is the state-of-the-art attack method. The result shows that AN-GCN achieves state-of-the-art accuracy under attack. When the perturbation increases, the accuracy gap between our method and other methods widens, clearly demonstrating the advantage of our method.
\begin{table}[htb]
% \vspace{-0.3cm}
    \centering
    \caption{Classification accuracy for various models after meta-learning attack}\label{tabaa}
\includegraphics[width=8.3cm]{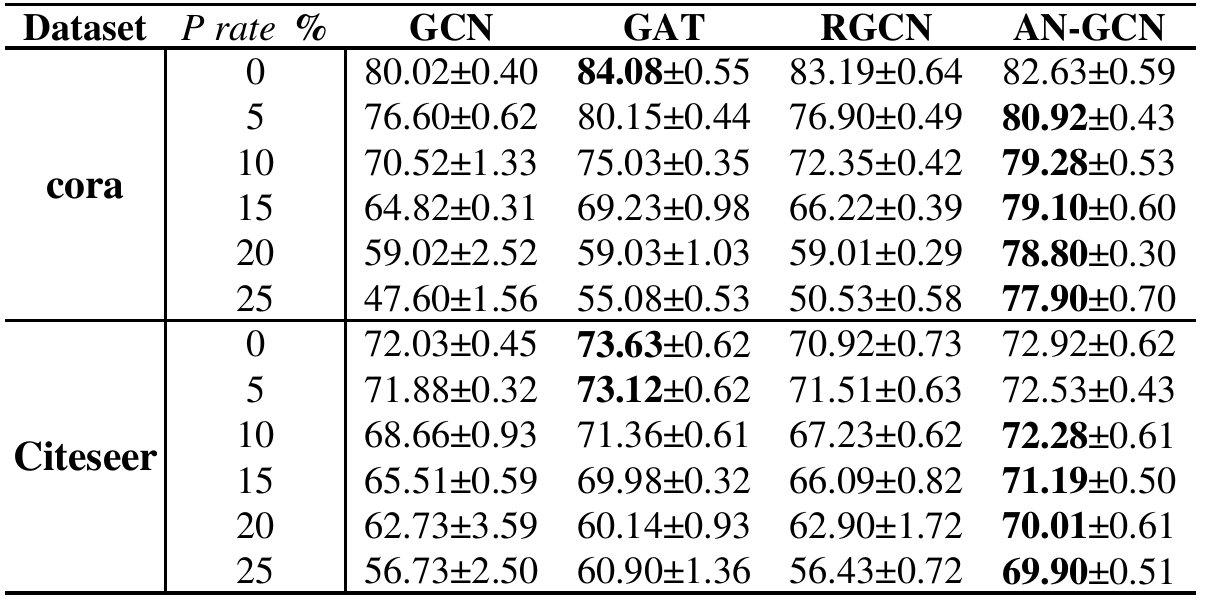}
% \vspace{-0.65cm}
\end{table}

% Thirdly, we evaluate the time efficiency of AN-GCN. Due to the introduction of Gan, it will cause inevitable training burden
In our previous work~\cite{69}, we have proved that in multi-class classification, using samples with other categories as the supervised samples of the generated samples will accelerate the convergence of GAN, similarly, the supervised samples of discriminator in AN-GCN are the same. Therefore, using GAN will not cause an excessive burden.

% $acc_G=\frac{\sum _ { \gamma = 1 } ^ { N } \zeta_{\gamma}^G}{N}$, $
% \zeta_{\gamma}^G = \left\{ \begin{array} { l } 1 , \operatorname { argmax } \left( D_{U^D} \left( G \left( Z _ { \gamma } \right) ; \theta ^ { D } \right) \right) = \operatorname { argmax } \left( \operatorname { label } _ { \gamma } \right) \\ 0 , \operatorname { argmax } \left( D_{U^D} \left( G \left( Z _ { \gamma } \right) ; \theta ^ { D } \right) \right) \neq \operatorname { argmax } \left( \operatorname { label } _ { \gamma } \right) \end{array} \right.$. At the same time, the accuracy of the discriminator is the accuracy of classifying nodes using $U$, $acc_D=\frac{\sum _ { \gamma = 1 } ^ { N } \zeta_{ \gamma }^D}{N}$, $\zeta_{ \gamma }^D = \left\{ \begin{array} { l } 1 , \operatorname { argmax } \left( D _ { U } \left( u ( \gamma ) ; \theta ^ { D } \right) \right) = \operatorname { argmax } \left( \operatorname { label } _ { \gamma } \right) \\ 0 , \operatorname { argmax } \left( D _ { U } \left( u ( \gamma ) ; \theta ^ { D } \right) \right) \neq \operatorname { argmax } \left( \operatorname { label } _ { \gamma } \right) \end{array} \right.$. We use  $acc_{GCN}$ to denote the accuracy of single-layer GCN (the convolution kernel adopts symmetric normalized Laplacian matrix~\cite{14}) with the same kernal of $D$ as a comparison. The experimental results are shown in Fig.~\ref{fig6}. 

\section{Conclusions}
In this paper, we solve the vulnerability problem of GCNs. We propose a general edge-perturbing attack model to demonstrate the vulnerability of GCNs as the edge-reading permission. Therefore, we propose an an Anonymous Graph Convolutional Network (AN-GCN) which is capable of withdrawing the edge-reading permission. We propose a node signal vibration model to represent the feature changes of the single node during the end-to-end training. Then, we introduce the node location theory, which makes the node position computable. Then we propose a generator with Staggered  Gaussian distribution as input, and improve the existing spectral convolution network to the discriminator. Our proposed AN-GCN does not need edge-reading permission when classifying nodes, so as to solve the vulnerability of GCNs.

\begin{spacing}{1}
\bibliographystyle{IEEEtran}
\bibliography{ref.bib}

% Generated by IEEEtran.bst, version: 1.14 (2015/08/26)
\begin{thebibliography}{10}
\providecommand{\url}[1]{#1}
\csname url@samestyle\endcsname
\providecommand{\newblock}{\relax}
\providecommand{\bibinfo}[2]{#2}
\providecommand{\BIBentrySTDinterwordspacing}{\spaceskip=0pt\relax}
\providecommand{\BIBentryALTinterwordstretchfactor}{4}
\providecommand{\BIBentryALTinterwordspacing}{\spaceskip=\fontdimen2\font plus
\BIBentryALTinterwordstretchfactor\fontdimen3\font minus
  \fontdimen4\font\relax}
\providecommand{\BIBforeignlanguage}[2]{{%
\expandafter\ifx\csname l@#1\endcsname\relax
\typeout{** WARNING: IEEEtran.bst: No hyphenation pattern has been}%
\typeout{** loaded for the language `#1'. Using the pattern for}%
\typeout{** the default language instead.}%
\else
\language=\csname l@#1\endcsname
\fi
#2}}
\providecommand{\BIBdecl}{\relax}
\BIBdecl

\bibitem{95}
K.~SE and K.~DL, ``Model uncertainty, political contestation, and public trust
  in science: Evidence from the covid-19 pandemic,'' \emph{Sci. Adv.}, vol.~6,
  no.~43, p. eabd4563, 2020.

\bibitem{96}
W.~Walt, C.~Jack, and T.~Christof, ``Adversarial explanations for understanding
  image classification decisions and improved neural network robustness,''
  \emph{Nat. Machine Intel.}, vol.~1, p. 508–516, 2019.

\bibitem{97}
F.~Samuel, G, B.~John, D, I.~Joichi, Z.~Jonathan, L, B.~Andrew, L, and
  K.~Isaac, S, ``Adversarial attacks on medical machine learning,''
  \emph{Science}, vol. 363, no. 6433, pp. 1287--1289, 2019.

\bibitem{53}
F.~Monti, F.~Frasca, D.~Eynard, D.~Mannion, and M.~M. Bronstein, ``Fake news
  detection on social media using geometric deep learning,'' \emph{arXiv
  preprint arxiv:1902.06673}, 2019.

\bibitem{54}
K.~Shu, D.~Mahudeswaran, S.~Wang, and H.~Liu, ``Hierarchical propagation
  networks for fake news detection: Investigation and exploitation,'' in
  \emph{Proc. 14th Int. AAAI Conf. Web Soc. Media}, vol.~14, 2020, pp.
  626--637.

\bibitem{89}
Z.~Li, X.~Dan, A.~Anurag, and S.~Philip, H, ``Dynamic graph message passing
  networks,'' in \emph{Proc. IEEE/CVF Conf. Comput. Vis. Pattern Recognit.},
  2020, pp. 3723--3732.

\bibitem{90}
I.~Valerii, H.~Markus, and L.~Harri, ``Learning continuous-time pdes from
  sparse data with graph neural networks,'' \emph{arxiv preprint
  arxiv:2006.08956}, 2020.

\bibitem{80}
C.~Jinyin, Z.~Jian, C.~Zhi, D.~Min, L.~Feifei, and Q.~Xuan, ``Time-aware
  gradient attack on dynamic network link prediction,'' \emph{arXiv preprint
  arXiv:1911.10561}, 2019.

\bibitem{81}
C.~Yizheng, Y.~Nadji, K.~Athanasios, M.~Fabian, P.~Roberto, A.~Manos, and
  N.~Vasiloglou, ``Practical attacks against graph-based clustering,'' in
  \emph{Proc. ACM Conf. Computer. Commun. Secur.}, 2017, p. 1125–1142.

\bibitem{19}
J.~Ma, S.~Ding, and Q.~Mei, ``Towards more practical adversarial attacks on
  graph neural networks,'' in \emph{Proc. 34th Adv. neural inf. proces. syst.},
  2020.

\bibitem{20}
Y.~Sun, S.~Wang, X.~Tang, T.-Y. Hsieh, and V.~Honavar, ``Non-target-specific
  node injection attacks on graph neural networks: A hierarchical reinforcement
  learning approach,'' in \emph{Proc. 29th Int. Conf. World Wide Web}, vol.~3,
  2020.

\bibitem{106}
S.~Ali, H.~W, Ronny, S.~Christoph, F.~Soheil, and G.~Tom, ``Are adversarial
  examples inevitable?'' in \emph{Proc. 7th Int. Conf. Learn. Represent.},
  2019.

\bibitem{55}
N.~Carlini and D.~Wagner, ``Adversarial examples are not easily detected:
  Bypassing ten detection methods,'' in \emph{Proc. 10th ACM workshop on
  artificial intelligence and security}, 2017, pp. 3--14.

\bibitem{70}
B.~Liang, H.~Li, M.~Su, X.~Li, W.~Shi, and X.~Wang, ``Detecting adversarial
  image examples in deep neural networks with adaptive noise reduction,''
  \emph{IEEE Trans. Dependable Secur. Comput.}, vol.~18, no.~1, pp. 72--85,
  2021.

\bibitem{72}
Z.~Che, A.~Borji, G.~Zhai, S.~Ling, J.~Li, X.~Min, G.~Guo, and P.~L. Callet,
  ``Smgea: A new ensemble adversarial attack powered by long-term gradient
  memories,'' \emph{IEEE Trans. Neural Netw. Learn. Syst.}, pp. 1--15, 2020.

\bibitem{73}
A.~Chaturvedi and U.~Garain, ``Mimic and fool: A task-agnostic adversarial
  attack,'' \emph{IEEE Trans. Neural Netw. Learn. Syst.}, vol.~32, no.~4, pp.
  1801--1808, 2021.

\bibitem{74}
Z.~Katzir and Y.~Elovici, ``Gradients cannot be tamed: Behind the impossible
  paradox of blocking targeted adversarial attacks,'' \emph{IEEE Trans. Neural
  Netw. Learn. Syst.}, vol.~32, no.~1, pp. 128--138, 2021.

\bibitem{75}
J.~Liu, N.~Akhtar, and A.~Mian, ``Adversarial attack on skeleton-based human
  action recognition,'' \emph{IEEE Trans. Neural Netw. Learn. Syst.}, pp.
  1--14, 2020.

\bibitem{76}
A.~Agarwal, G.~Goswami, M.~Vatsa, R.~Singh, and N.~K. Ratha, ``Damad: Database,
  attack, and model agnostic adversarial perturbation detector,'' \emph{IEEE
  Trans. Neural Netw. Learn. Syst.}, pp. 1--13, 2021.

\bibitem{82}
W.~Y, Z.~D, Y.~J, B.~J, M.~X, and Q.~Gu, ``Improving adversarial robustness
  requires revisiting misclassified examples,'' \emph{8th Int. Conf. Learn.
  Represent.}, 2020.

\bibitem{83}
S.~C, H.~K, L.~J, W.~L, and H.~J, E, ``Robust local features for improving the
  generalization of adversarial training,'' in \emph{Proc. 7th Int. Conf.
  Learn. Represent.}, 2019.

\bibitem{85}
W.~T, T.~L, and V.~Y, ``Defending against physically realizable attacks on
  image classification,'' in \emph{Proc. 8th Int. Conf. Learn. Represent.},
  2020.

\bibitem{86}
C.~H, Z.~H, B.~D, and H.~C, J, ``Robust decision trees against adversarial
  examples,'' in \emph{Proc. 36th Int. Conf. Mach. Learn.}, 2019, p.
  1122–1131.

\bibitem{77}
Z.~Kai, M.~Tomasz, P, and V.~Yevgeniy, ``Adversarial robustness of
  similarity-based link prediction,'' in \emph{Proc. 19th IEEE Int. Conf. Data
  Min.}, 2019, pp. 1--13.

\bibitem{78}
X.~Kaidi, C.~Hongge, L.~Sijia, C.~Pin~Yu, W.~Tsui-Wei, H.~Mingyi, and X.~Lin,
  ``Topology attack and defense for graph neural networks: an optimization
  perspective,'' in \emph{Proc. 28th Int. Joint Conf. Artif. Intell.}, 2019, p.
  3961–3967.

\bibitem{79}
D.~Quanyu, S.~Xiao, Z.~Liang, L.~Qiang, and D.~Wang, ``Adversarial training
  methods for network embedding,'' in \emph{Proc. 28th Int. Conf. World Wide
  Web}, 2019, p. 329–339.

\bibitem{31}
W.~L. Hamilton, R.~Ying, and J.~Leskovec, ``Inductive representation learning
  on large graphs,'' in \emph{Proc. 31th Adv. neural inf. proces. syst.}, 2017,
  pp. 1025--1035.

\bibitem{94}
K.-H. Lai, D.~Zha, K.~Zhou, and X.~Hu, ``Policy-gnn: Aggregation optimization
  for graph neural networks,'' in \emph{Proc. 26th ACM SIGKDD Int. Conf. Knowl.
  Discov. Data Min.}, 2020, p. 461–471.

\bibitem{87}
G.~Simon, Z.~Daniel, and G.~Stephan, ``Reliable graph neural networks via
  robust aggregation,'' in \emph{Proc. 34th Adv. neural inf. proces. syst.},
  2020.

\bibitem{88}
Z.~Xiang and Z.~Marinka, ``Gnnguard: Defending graph neural networks against
  adversarial attacks,'' in \emph{Proc. 34th Adv. neural inf. proces. syst.},
  2020.

\bibitem{89_5}
Y.~Yuning, C.~Tianlong, S.~Yongduo, C.~Ting, W.~Zhangyang, and S.~Yang, ``Graph
  contrastive learning with augmentations,'' in \emph{Proc. 34th Adv. neural
  inf. proces. syst.}, 2020, pp. 3723--3732.

\bibitem{38}
D.~Zhu, Z.~Zhang, P.~Cui, and W.~Zhu, ``Robust graph convolutional networks
  against adversarial attacks,'' in \emph{Proc. 25th ACM SIGKDD Int. Conf.
  Knowl. Discov. Data Min.}, 2019, pp. 1399--1407.

\bibitem{1}
H.~Dai, H.~Li, T.~Tian, X.~Huang, L.~Wang, J.~Zhu, and L.~Song, ``Adversarial
  attack on graph structured data,'' in \emph{Proc. 35th Int. Conf. Mach.
  Learn.}, 2018, pp. 1115--1124.

\bibitem{2}
D.~Z{\"u}gner, A.~Akbarnejad, and S.~G{\"u}nnemann, ``Adversarial attacks on
  neural networks for graph data,'' in \emph{Proc. 24th ACM SIGKDD Int. Conf.
  Knowl. Discov. Data Min.}, 2018, pp. 2847--2856.

\bibitem{98}
C.~Sizhe, H.~Zhengbao, S.~Chengjin, Y.~Jie, and H.~Xiaolin, ``Universal
  adversarial attack on attention and the resulting dataset damagenet,''
  \emph{IEEE Trans. Pattern Anal. Mach. Intell.}, 2020.

\bibitem{99}
L.~Yen-Chen, H.~Zhang-Wei, L.~Yuan-Hong, S.~Meng-Li, L.~Ming-Yu, and M.~Sun,
  ``Tactics of adversarial attack on deep reinforcement learning agents,'' in
  \emph{Proc. 26th Int. Joint Conf. Artif. Intell.}, 2017.

\bibitem{91}
Z.~Lingxiao and A.~Leman, ``Pairnorm: Tackling oversmoothing in gnns,'' in
  \emph{Proc. 8th Int. Conf. Learn. Represent.}, 2020.

\bibitem{92}
L.~Guohao, M.~Matthias, T.~Ali, and G.~Bernard, ``Deepgcns: Can gcns go as deep
  as cnns?'' in \emph{Proc. 17th IEEE. Int. Conf. Comput. Vision.}, 2019.

\bibitem{93}
L.~Qimai, H.~Zhichao, and W.~Xiao-Ming, ``Deeper insights into graph
  convolutional networks for semi-supervised learning,'' in \emph{Proc. 32th
  AAAI Conf. Artif. Intell.}, 2018, pp. 3538--3545.

\bibitem{100}
D.~Wang, J.~Lin, P.~Cui, Q.~Jia, Z.~Wang, Y.~Fang, Q.~Yu, J.~Zhou, S.~Yang, and
  Y.~Qi, ``A semi-supervised graph attentive network for financial fraud
  detection,'' in \emph{Proc. 19th IEEE Int. Conf. Data Min.}, 2019, pp.
  598--607.

\bibitem{101}
I.~Alqassem, I.~Rahwan, and D.~Svetinovic, ``The anti-social system properties:
  Bitcoin network data analysis,'' \emph{IEEE Trans. Syst. Man Cybern.},
  vol.~50, no.~1, pp. 21--31, 2020.

\bibitem{102}
Z.~He, P.~Chen, X.~Li, Y.~Wang, G.~Yu, C.~Chen, X.~Li, and Z.~Zheng, ``A
  spatiotemporal deep learning approach for unsupervised anomaly detection in
  cloud systems,'' \emph{IEEE Trans. Neural Netw. Learn. Syst.}, pp. 1--15,
  2020.

\bibitem{103}
Z.~Li, L.~Zhenpeng, L.~Jian, L.~Zhao, and G.~Jun, ``Addgraph: Anomaly detection
  in dynamic graph using attention-based temporal gcn,'' in \emph{Proc. 28th
  Int. Joint Conf. Artif. Intell.}, vol. 2019, 2019, pp. 4419--4425.

\bibitem{104}
L.~Y, Y.~R, S.~C, and L.~Y, ``Diffusion convolutional recurrent neural network:
  Data-driven traffic forecasting,'' in \emph{Proc. 6th Int. Conf. Learn.
  Represent.}, 2018.

\bibitem{105}
Y.~S, X.~Y, and L.~D, ``Spatial temporal graph convolutional networks for
  skeleton-based action recognition,'' in \emph{Proc. 32nd AAAI Conf. Artif.
  Intell}, 2018.

\bibitem{22}
M.~Defferrard, X.~Bresson, and P.~Vandergheynst, ``Convolutional neural
  networks on graphs with fast localized spectral filtering,'' in \emph{Proc.
  30th Adv. neural inf. proces. syst.}, 2016, pp. 3844--3852.

\bibitem{3}
H.~Chang, Y.~Rong, T.~Xu, W.~Huang, H.~Zhang, P.~Cui, W.~Zhu, and J.~Huang, ``A
  restricted black-box adversarial framework towards attacking graph embedding
  models,'' in \emph{Proc. 34th AAAI Conf. Artif. Intell.}, vol.~34, no.~04,
  2020, pp. 3389--3396.

\bibitem{4}
A.~Bojchevski and S.~G{\"u}nnemann, ``Adversarial attacks on node embeddings
  via graph poisoning,'' in \emph{Proc. 36th Int. Conf. Mach. Learn.}, 2019,
  pp. 695--704.

\bibitem{5}
B.~Wang and N.~Z. Gong, ``Attacking graph-based classification via manipulating
  the graph structure,'' in \emph{Proc. ACM Conf. Computer. Commun. Secur.},
  2019, pp. 2023--2040.

\bibitem{71}
X.~Zhaohan, P.~Ren, J.~Shouling, and W.~Ting, ``Graph backdoor,'' in
  \emph{Proc. 29th USENIX Secur. Symp.}, 2021.

\bibitem{6}
X.~Tang, Y.~Li, Y.~Sun, H.~Yao, P.~Mitra, and S.~Wang, ``Transferring
  robustness for graph neural network against poisoning attacks,'' in
  \emph{Proc. 13th ACM Int. Conf. Web Search Data Min.}, 2020, pp. 600--608.

\bibitem{7}
M.~Jin, H.~Chang, W.~Zhu, and S.~Sojoudi, ``Power up! robust graph
  convolutional network against evasion attacks based on graph powering,''
  \emph{arXiv preprint arxiv:1905.10029}, 2019.

\bibitem{8}
D.~Z{\"u}gner and S.~G{\"u}nnemann, ``Certifiable robustness and robust
  training for graph convolutional networks,'' in \emph{Proc. 25th ACM SIGKDD
  Int. Conf. Knowl. Discov. Data Min.}, 2019, pp. 246--256.

\bibitem{9}
Z.~Deng, Y.~Dong, and J.~Zhu, ``Batch virtual adversarial training for graph
  convolutional networks,'' \emph{arXiv preprint arxiv:1902.09192}, 2019.

\bibitem{10}
S.~Wang, Z.~Chen, J.~Ni, X.~Yu, Z.~Li, H.~Chen, and P.~S. Yu, ``Adversarial
  defense framework for graph neural network,'' \emph{arXiv preprint
  arxiv:1905.03679}, 2019.

\bibitem{58}
F.~Feng, X.~He, J.~Tang, and T.-S. Chua, ``Graph adversarial training:
  Dynamically regularizing based on graph structure,'' \emph{IEEE Trans. Knowl.
  Data Eng.}, 2019.

\bibitem{34}
Z.~Wu, S.~Pan, F.~Chen, G.~Long, C.~Zhang, and S.~Y. Philip, ``A comprehensive
  survey on graph neural networks,'' \emph{IEEE Trans. Neural Netw. Learn.
  Syst.}, 2020.

\bibitem{25}
P.~Cui, X.~Wang, J.~Pei, and W.~Zhu, ``A survey on network embedding,''
  \emph{IEEE Trans. Knowl. Data Eng.}, vol.~31, no.~5, pp. 833--852, 2018.

\bibitem{30}
W.~Huang, T.~Zhang, Y.~Rong, and J.~Huang, ``Adaptive sampling towards fast
  graph representation learning,'' in \emph{Proc. 32th Adv. neural inf. proces.
  syst.}, vol.~31, 2018, pp. 4558--4567.

\bibitem{26}
B.~Perozzi, R.~Al-Rfou, and S.~Skiena, ``Deepwalk: Online learning of social
  representations,'' in \emph{Proc. 20th ACM SIGKDD Int. Conf. Knowl. Discov.
  Data Min.}, 2014, pp. 701--710.

\bibitem{27}
A.~Grover and J.~Leskovec, ``node2vec: Scalable feature learning for
  networks,'' in \emph{Proc. 22th ACM SIGKDD Int. Conf. Knowl. Discov. Data
  Min.}, 2016, pp. 855--864.

\bibitem{28}
J.~Tang, M.~Qu, M.~Wang, M.~Zhang, J.~Yan, and Q.~Mei, ``Line: Large-scale
  information network embedding,'' in \emph{Proc. 24th Int. Conf. World Wide
  Web}, 2015, pp. 1067--1077.

\bibitem{29}
C.~Yang, Z.~Liu, D.~Zhao, M.~Sun, and E.~Y. Chang, ``Network representation
  learning with rich text information.'' in \emph{Proc. 24th Int. Joint Conf.
  Artif. Intell.}, vol. 2015, 2015, pp. 2111--2117.

\bibitem{32}
P.~Veli{\v{c}}kovi{\'c}, G.~Cucurull, A.~Casanova, A.~Romero, P.~Lio, and
  Y.~Bengio, ``Graph attention networks,'' in \emph{Proc. 5th Int. Conf. Learn.
  Represent.}, 2017.

\bibitem{14}
T.~N. Kipf and M.~Welling, ``Semi-supervised classification with graph
  convolutional networks,'' in \emph{Proc. 5th Int. Conf. Learn. Represent.},
  2017.

\bibitem{24}
K.~Zhou, Y.~Dong, W.~S. Lee, B.~Hooi, H.~Xu, and J.~Feng, ``Effective training
  strategies for deep graph neural networks,'' \emph{arXiv preprint
  arxiv:2006.07107}, 2020.

\bibitem{21}
A.~Loukas, ``What graph neural networks cannot learn: depth vs width,'' in
  \emph{Proc. 7th Int. Conf. Learn. Represent.}, 2019.

\bibitem{35}
L.~Sun, Y.~Dou, C.~Yang, J.~Wang, P.~S. Yu, L.~He, and B.~Li, ``Adversarial
  attack and defense on graph data: A survey,'' \emph{arXiv preprint
  arxiv:1812.10528}, 2018.

\bibitem{16}
A.~Bojchevski and S.~G{\"u}nnemann, ``Adversarial attacks on node embeddings
  via graph poisoning,'' in \emph{Proc. 36th Int. Conf. Mach. Learn.}, 2019,
  pp. 695--704.

\bibitem{17}
M.~Fang, G.~Yang, N.~Z. Gong, and J.~Liu, ``Poisoning attacks to graph-based
  recommender systems,'' in \emph{Proc. 34th Annual Computer Secur. Applicat.
  Conf.}, 2018, pp. 381--392.

\bibitem{18}
X.~Liu, S.~Si, X.~Zhu, Y.~Li, and C.-J. Hsieh, ``A unified framework for data
  poisoning attack to graph-based semi-supervised learning,'' in \emph{Proc.
  33th Adv. neural inf. proces. syst.}, 2019.

\bibitem{56}
D.~I. Shuman, S.~K. Narang, P.~Frossard, A.~Ortega, and P.~Vandergheynst, ``The
  emerging field of signal processing on graphs: Extending high-dimensional
  data analysis to networks and other irregular domains,'' \emph{IEEE Signal
  Process. Mag.}, vol.~30, no.~3, pp. 83--98, 2013.

\bibitem{65}
S.~Khorasani and A.~Adibi, ``Analytical solution of linear ordinary
  differential equations by differential transfer matrix method.''
  \emph{Electron. J. Differ. Equ.}, 2003.

\bibitem{66}
A.~Nan, M.~Tennant, U.~Rubin, and N.~Ray, ``Ordinary differential equation and
  complex matrix exponential for multi-resolution image registration,''
  \emph{arXiv preprint arxiv:2007.13683}, 2020.

\bibitem{51}
S.-C. Pei, W.-L. Hsue, and J.-J. Ding, ``Discrete fractional fourier transform
  based on new nearly tridiagonal commuting matrices,'' \emph{IEEE Trans.
  Signal Process.}, vol.~54, no.~10, pp. 3815--3828, 2006.

\bibitem{67}
Z.~Chen, A.~Luo, and X.~Huang, ``Euler's formula-based stability analysis for
  wideband harmonic resonances,'' \emph{IEEE Trans. Ind. Electron.}, vol.~67,
  no.~11, pp. 9405--9417, 2019.

\bibitem{52}
J.~Bruna, W.~Zaremba, A.~Szlam, and Y.~LeCun, ``Spectral networks and locally
  connected networks on graphs,'' in \emph{Proc. 2nd Int. Conf. Learn.
  Represent.}, 2014.

\bibitem{62}
A.~Bojchevski and S.~G{\"u}nnemann, ``Deep gaussian embedding of graphs:
  Unsupervised inductive learning via ranking,'' in \emph{Proc. 5th Int. Conf.
  Learn. Represent.}, 2017.

\bibitem{23}
P.~Sen, G.~Namata, M.~Bilgic, L.~Getoor, B.~Galligher, and T.~Eliassi-Rad,
  ``Collective classification in network data,'' \emph{AI Mag.}, vol.~29,
  no.~3, pp. 93--93, 2008.

\bibitem{64}
C.~L. Giles, K.~D. Bollacker, and S.~Lawrence, ``Citeseer: An automatic
  citation indexing system,'' in \emph{Proc. 3rd ACM. Int. Conf. Digital.
  Libr.}, 1998, pp. 89--98.

\bibitem{59}
A.~Bojchevski and S.~G\"{u}nnemann, ``Certifiable robustness to graph
  perturbations,'' in \emph{Proc. 33th Adv. neural inf. proces. syst.},
  vol.~32, 2019.

\bibitem{61}
A.~Bojchevski and S.~G{\"u}nnemann, ``Adversarial attacks on node embeddings
  via graph poisoning,'' in \emph{Proc. 36th Int. Conf. Mach. Learn.}, 2019,
  pp. 695--704.

\bibitem{37}
W.~Jin, Y.~Ma, X.~Liu, X.~Tang, S.~Wang, and J.~Tang, ``Graph structure
  learning for robust graph neural networks,'' in \emph{Proc. 26th ACM SIGKDD
  Int. Conf. Knowl. Discov. Data Min.}, 2020, pp. 66--74.

\bibitem{68}
L.~Wang, Y.~Zhang, and J.~Feng, ``On the euclidean distance of images,''
  \emph{IEEE Trans. Pattern Anal. Mach. Intell.}, vol.~27, no.~8, pp.
  1334--1339, 2005.

\bibitem{36}
D.~Z{\"u}gner and S.~G{\"u}nnemann, ``Adversarial attacks on graph neural
  networks via meta learning,'' in \emph{Proc. 6th Int. Conf. Learn.
  Represent.}, 2018.

\bibitem{69}
A.~Liu, Y.~Wang, and T.~Li, ``{SFE-GACN}: A novel unknown attack detection
  under insufficient data via intra categories generation in embedding space,''
  \emph{Comput. Secur.}, p. 102262, 2021.

\end{thebibliography}
\end{spacing}

% \section{References}

% List and number all bibliographical references in 9-point Times,
% single-spaced, at the end of your paper. When referenced in the text,
% enclose the citation number in square brackets, for
% example~\cite{Authors14}.  Where appropriate, include the name(s) of
% editors of referenced books.

%-------------------------------------------------------------------------

% For peer review papers, you can put extra information on the cover
% page as needed:
% \ifCLASSOPTIONpeerreview
% \begin{center} \bfseries EDICS Category: 3-BBND \end{center}
% \fi
%
% For peerreview papers, this IEEEtran command inserts a page break and
% creates the second title. It will be ignored for other modes.
\IEEEpeerreviewmaketitle

% %

% \appendices
% \section{Proof of the First Zonklar Equation}
% Appendix one text goes here.

% % use section* for acknowledgment
% \ifCLASSOPTIONcompsoc
%   % The Computer Society usually uses the plural form
%   \section*{Acknowledgments}
% \else
%   % regular IEEE prefers the singular form
%   \section*{Acknowledgment}
% \fi

% The authors would like to thank...

% % Can use something like this to put references on a page
% % by themselves when using endfloat and the captionsoff option.
% \ifCLASSOPTIONcaptionsoff
%   \newpage
% \fi

% \begin{thebibliography}{1}

% \bibitem{IEEEhowto:kopka}
% H.~Kopka and P.~W. Daly, \emph{A Guide to \LaTeX}, 3rd~ed.\hskip 1em plus
%   0.5em minus 0.4em\relax Harlow, England: Addison-Wesley, 1999.

% \end{thebibliography}

% \begin{IEEEbiography}{Michael Shell}
% Biography text here.
% \end{IEEEbiography}

% \begin{IEEEbiographynophoto}{John Doe}
% Biography text here.
% \end{IEEEbiographynophoto}

% \begin{IEEEbiographynophoto}{Jane Doe}
% Biography text here.
% \end{IEEEbiographynophoto}

\end{document}